\documentclass[11pt,letterpaper]{article}

\usepackage[margin=1in]{geometry}
\usepackage[T1]{fontenc}
\usepackage{newtxtext}
\usepackage{helvet}
\usepackage{courier}
\usepackage[hyphens]{url}
\usepackage{graphicx}
\usepackage{natbib}
\usepackage{caption}
\usepackage{algorithm}

\usepackage{newfloat}
\usepackage{listings}
\DeclareCaptionStyle{ruled}{labelfont=normalfont,labelsep=colon,strut=off}
\floatstyle{ruled}
\newfloat{listing}{tb}{lst}{}
\floatname{listing}{Listing}

\usepackage{booktabs}

\usepackage{amsmath,amsfonts}
\usepackage{amsthm}
\usepackage[inline]{enumitem}

\usepackage{amsmath,amsfonts,bm}

\def\eqref#1{equation~\ref{#1}}

\def\1{\bm{1}}

\DeclareMathAlphabet{\mathsfit}{\encodingdefault}{\sfdefault}{m}{sl}
\SetMathAlphabet{\mathsfit}{bold}{\encodingdefault}{\sfdefault}{bx}{n}

\DeclareMathOperator*{\argmin}{arg\,min}

\usepackage{hyperref}
\usepackage{url}
\usepackage{amsmath,amssymb,amsthm}
\usepackage{booktabs}
\usepackage{algorithm}
\usepackage{algpseudocode}
\usepackage{graphicx}
\usepackage{wrapfig}
\usepackage{subcaption}
\usepackage{enumitem}
\usepackage[table]{xcolor}
\usepackage{booktabs}
\usepackage{multirow}

\newtheorem{definition}{Definition}

\newtheorem{proposition}{Proposition}

\newtheorem{remark}{Remark}

\newcommand{\Reals}{\mathbb{R}}

\title{Representation Alignment from a Few Paired Data via Distilled Optimal Transport}

\author{
  Julien Kordes, Reuben Dorent\\[6pt]
  \parbox{0.85\textwidth}{\centering\small
  Sorbonne Universit\'e, Institut du Cerveau - Paris Brain Institute - ICM, CNRS, Inria, Inserm, AP-HP, H\^opital de la Piti\'e Salp\^etri\`ere, 75013 Paris, France\\
  \texttt{reuben.dorent@inria.fr}}
}
\date{}

\begin{document}

\maketitle

\begin{abstract}
Aligning two independently pretrained unimodal encoders typically demands millions of paired examples. In this paper, we study the regime where only $N \sim 100$--$500$ pairs are available, while unpaired data on either side is abundant. Our method expresses every representation by its similarities to the $N$ paired anchors, which places both modalities in a common coordinate system; matches the two point clouds with a single entropic optimal transport problem solved over the entire pool, whose marginal constraints counteract the hubness that defeats nearest-neighbor matching; and distills the resulting plan into two lightweight encoders, so that the correspondence becomes inductive. We additionally put this marginal constraint in relation with the one-step hub corrections of the bilingual lexicon induction literature.
Pairing is required only for the anchors: the unpaired collections may be drawn independently from the two marginals, so that no item outside the anchor set needs to have both of its modalities observed. We evaluate on image-text and molecule-text alignment, where our method outperforms existing approaches at the same anchor budget on Flickr30k, MS-COCO and ChEBI-20. 
\end{abstract}

\section{Introduction}

Vision-language models such as CLIP \citep{radford2021clip} learn a shared cross-modal space from hundreds of millions of paired examples. However, such paired multimodal data is scarce, requiring costly manual annotation or specialized paired measurements. In contrast, unimodal data can be abundant, and strong unimodal encoders are readily available off the shelf \citep{caron2021dino, reimers2019sbert, ross2022molformer}. Leveraging these encoders for multimodal tasks requires aligning their embedding space, a fundamental challenge: trained independently, on different objectives and data distributions, their spaces differ in dimension, scale, and geometry, with no known mapping between them.  In this work, we investigate the extreme low-resource regime: aligning two pre-trained unimodal encoders using only a few hundred paired examples ($N \sim 100\text{--}500$).

Prior approaches to low-resource cross-modal alignment fall into three categories.
1) Anchor-relative methods express each representation as its vector of similarities to a set of \emph{anchors}, i.e. paired multimodal representations, casting both modalities into a shared coordinate system. 
\citet{moschella2022relative} introduced this approach for a single modality, and ASIF \citep{norelli2023asif} extended across modalities. However, these non-parametric methods require very large anchor collections, e.g. $1.6\times10^6$ anchors for ASIF. Moreover, inference is costly since the large anchor set must be stored to compute the anchor-relative coordinates.
2)~Graph-based methods such as SUE \citep{yacobi2025sue} embed each modality independently through the spectrum of its neighborhood graph, whose edges are weighted by intra-modal distances, and use the anchors only to provide the initial linear alignment. While SUE operates at a low anchor budget ($N \approx 500$), it assumes that independently trained unimodal encoders converge toward latent spaces with comparable local geometry, a local form of the Platonic Representation Hypothesis \citep{huh2024platonic}. This assumption is strong: cross-modal representations agree locally on which points are close, but not on local distances  \citep{groger2026aristotelian}. 
3) Recently, 
SOTAlign \citep{roschmann2026sotalign} also performs a linear alignment using the anchors and uses it as a teacher, regularizing a student on unpaired batches with an optimal-transport divergence derived from the linear alignment. Since this regularization enforces consistency with the teacher, SOTAlign inherits the limits of  the linear alignment.

To overcome these limitations, we introduce AnchorDOT, a three-stage alignment framework (Figure~\ref{fig:overview}) that combines (i) the expressiveness of anchor-relative coordinates with (ii) the global matching guarantees of optimal transport~(OT) and (iii) the inference efficiency of parametric distillation. First, we project unpaired collections from both modalities into a common anchor-relative coordinate space. Next, rather than performing nearest-neighbor matching, which suffers from the \emph{hubness}  problem in which a small number of embeddings are the nearest neighbors of many queries~\citep{radovanovic2010hubs, dinu2014improving}, we formulate cross-modal alignment as a global entropic optimal transport problem over the unpaired pool. The transport problem's marginal constraints determine the total mass each point receives, avoiding hubs. Finally, because the optimal transport plan is transductive and defined only on the unpaired pool, we distill the computed assignment plan into two compact, modality-specific student encoders. At inference time, new queries are embedded in a single forward pass, with no dependence on stored memory banks.

Our contributions can be summarized as follows: (i) we align two frozen unimodal encoders from $N \sim 100$--$500$ paired anchors and distill the resulting transport plan into inductive encoders that require no reference collection at inference; (ii) we show that pairing is needed only for the anchors, i.e. the two unpaired collections may be sampled independently from their marginals, with no latent pairing in the pool; and (iii) holding the anchor-relative cost fixed, we show that enforcing both transport marginals leads to more accurate and significantly less hub-dominated correspondences than nearest neighbor and the one-step corrections used in retrieval.

\begin{figure}[tb]
    \centering
    \includegraphics[width=0.99\linewidth]{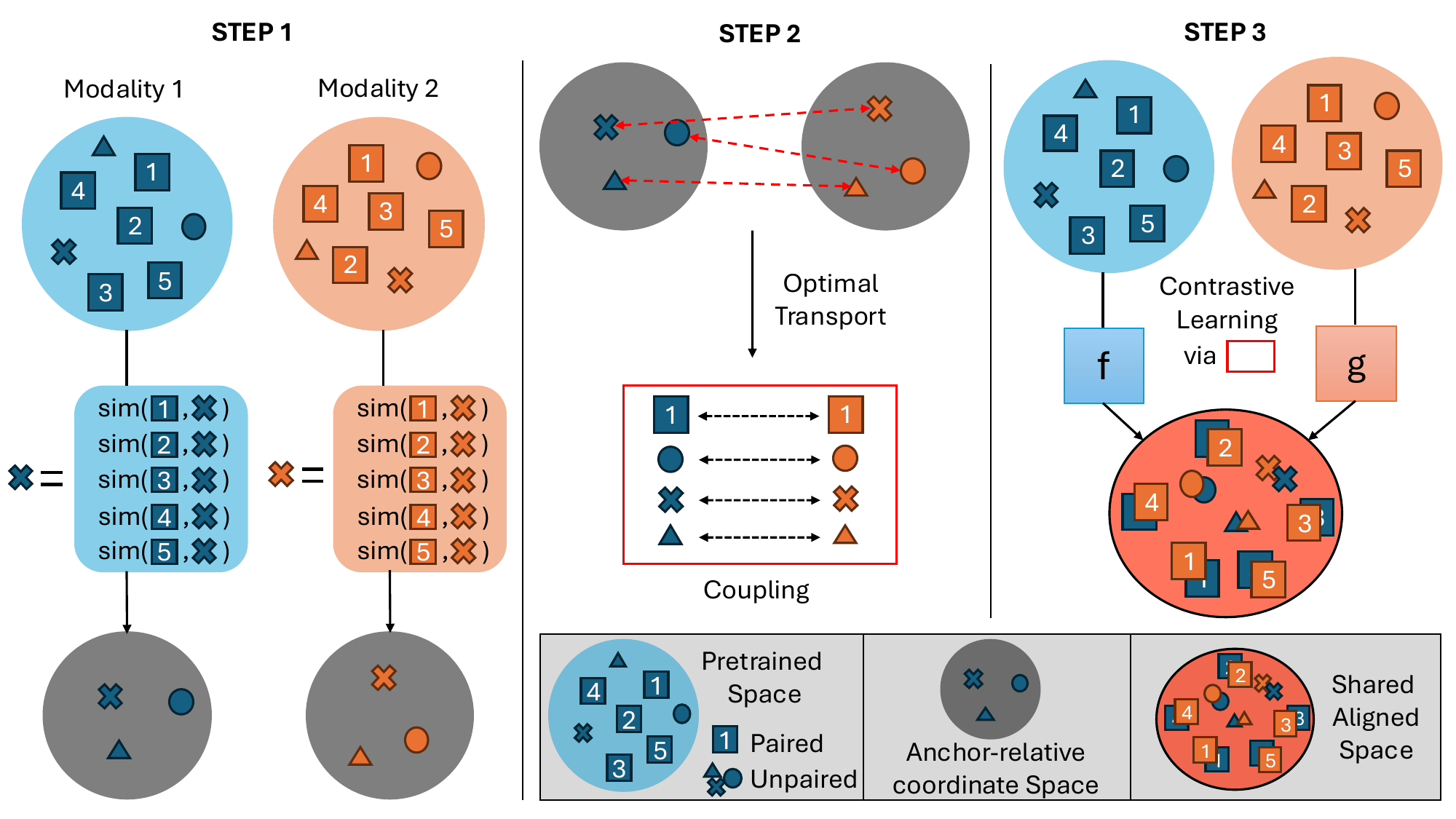}
    \caption{\textbf{Overview of AnchorDOT.} \textbf{Step 1:} modality-specific representations are expressed by their \textit{similarity profile} to a small set of shared anchors, embedding them in a common \emph{anchor-relative coordinate space}. \textbf{Step 2:} an entropic optimal transport coupling is solved between the two anchor-relative embedding spaces, recovering a soft cross-modal correspondence between points. \textbf{Step 3:} this coupling is used as a pseudo-pairing signal to train two lightweight encoders $f_\theta$ and $g_\psi$, via contrastive learning directly on the original pretrained embeddings, distilling the transport plan into a \emph{shared aligned space} that requires no anchors and no optimal transport at inference time.
    }
    \label{fig:overview}
\end{figure}

\section{Related work}
\label{sec:related}
\noindent \textbf{Anchor-relative representations.} \citet{moschella2022relative} represent a point by its cosine similarities to a fixed anchor set, which makes the representation invariant to angle-preserving transformations. Since networks trained on the same data from different seeds differ approximately by such transformations, this
enables zero-shot model stitching within a single modality. Similarly, \citet{cannistraci2023bootstrapping} discover new anchors by solving a Sinkhorn problem between relative representations at each optimization step. Their method is validated only on same-modality and cross-lingual communication. Their coupling is used to discover anchors rather than to align. Relative representations have since been used across modalities. ASIF \citep{norelli2023asif} builds a common space without training and performs retrieval via a sparsified nearest-neighbor rule. Because each query is decided independently against the reference collection, the collection has to describe the space densely enough for the correct neighbor to be present and to match on raw similarity: ASIF uses $1.6$ million reference pairs, and at $10^4$ pairs reaches only $18\%$ zero-shot ImageNet accuracy. It is also non-parametric, i.e. the whole collection is retained at inference. RELIT \citep{chen2023relit} builds the same anchor-relative representation, but only to mine pseudo-pairs by nearest-neighbor retrieval for pre-training a vision-language model, which is then fine-tuned on each task's full labelled data (the full Flickr30k training split, ${\sim}29$k images). Our $N$
anchors are instead the only paired supervision, and our matching is solved jointly over the pool. Concurrently, \citet{kim2026relative} align frozen encoders under scarce paired data using token-level relative representations with learnable anchors, which are optimized such that cross-modal similarity patterns agree. Our anchors are instead a given set of paired examples, and the matching is a transport problem rather than a similarity-matching objective. Outside vision-language, \citet{gu2022keypoint} guide an OT problem in heterogeneous domain adaptation with a relation profile over a handful of keypoints. Appendix~\ref{sec:KPG-RL} compares its construction to ours in detail. In summary, anchor-relative methods work with a few hundred anchors within a single modality, but have so far relied on far more paired data across modalities.

\noindent \textbf{Alignment with little or no supervision.} SUE \citep{yacobi2025sue} works at a comparable pair budget to ours but along a different axis: it computes a spectral embedding of each modality's random-walk graph independently, uses the pairs only to fit a linear alignment between the two spectral embeddings via Canonical Correlation Analysis (CCA), and refines with an MMD residual network. Its premise is that the two point clouds carry similar intra-domain structure, so the pairs are asked only to fix an ambiguity between two spaces that already correspond; where that structural correspondence is weak, little supervision remains to recover it. CSA \citep{li2025csa} applies CCA directly to the paired unimodal features and compares items by a cosine weighted by the canonical correlations. It trains no network, but fits on the pairs alone, leaves the unpaired data unused, and is validated with thousands of pairs ($5$k on Flickr30k), an order of magnitude above our budget. SOTAlign \citep{roschmann2026sotalign} fits a linear teacher on the same few pairs and regularizes a student network toward it on unpaired batches via an OT-based divergence. SOTAlign thus refines an existing alignment rather than discovering one, unlike ours. \citet{gholamzadeh2025model} build a cross-modal cost from intra-modal distances routed through the known pairs, an inter-modal bridge cost, which consults the anchors only through a per-pair minimum and so discards the profile of similarities to all anchors. At tens of thousands of pairs, two to nearly three orders of magnitude above the budget we target, STRUCTURE \citep{groger2025structure} aligns frozen encoders by regularizing towards each encoder's neighborhood geometry, using intra-modal structure as a penalty rather than to build the cost. With no pairs at all, Gromov-Wasserstein alignment \citep{alvarezmelis2018gromov} matches points by comparing intra-domain distances, MUSE \citep{conneau2017word} aligns adversarially, and Wasserstein Procrustes \citep{grave2019unsupervised} jointly estimates an orthogonal and a permutation matrix initialized by a convex relaxation. Its dictionary-seeded counterpart is the offline setting of \citet{smith2017offline}. All of these were developed and validated on two instances of the same modality (word embedding spaces of different languages). Gromov-Wasserstein does not require a shared dimension, but it matches on intra-domain geometry alone.

\noindent \textbf{Optimal transport as a distillation target.} Using a Sinkhorn coupling as soft supervision for a network has already been  proposed. OTTER \citep{wu2021data} computes an entropic plan within each batch of noisy web-collected image-text pairs and distils it as soft labels for contrastive pre-training. SwAV \citep{caron2020swav} enforces a Sinkhorn equipartition on cluster assignments within each batch, using the code of one augmented view to supervise the prediction from the other. S-CLIP \citep{mo2023sclip} transports unpaired images onto a few captioned images to build soft caption targets for CLIP. In each case, the similarity is available by construction: OTTER's teacher encoders are trained on paired data, SwAV compares two augmentations of a single image, and S-CLIP compares images with images in a pretrained CLIP space. 
Between independently trained unimodal encoders no such similarity exists, and the anchor-relative coordinates supply it.

\noindent \textbf{Hubness.} In high dimension, a few points, the hubs, are the nearest neighbors of many others \citep{radovanovic2010hubs}. The usual corrections are one-step rescalings of the similarity matrix, applied at retrieval time \citep{smith2017offline, conneau2017word, bogolin2022qbnorm}. Enforcing transport marginals instead constrains all assignments jointly. Section~\ref{sec:ot} relates the two.

\section{Method}
\label{sec:method}

\subsection{Problem setup and notation}

Let $x \in \Reals^{d_x}$ and $y \in \Reals^{d_y}$ denote representations obtained from two pre-trained, modality-specific encoders, and let $p(x, y)$ denote their joint distribution, with marginals $p(x)$ and $p(y)$. Our supervision is a set of $N$ \emph{paired anchors} drawn from the joint distribution,
\begin{equation}
    \mathcal{A} = \{(x_{a_k}, y_{a_k})\}_{k=1}^{N} \sim p(x,y)^{\otimes N}, \qquad N \in [100, 500].
\end{equation}
Beyond the anchors we assume access to two \emph{unpaired} collections drawn independently from the two \emph{marginals},
\begin{equation}
    \mathcal{X}_u = \{x_i\}_{i=1}^{n} \sim p(x)^{\otimes n}, \qquad
    \mathcal{Y}_u = \{y_j\}_{j=1}^{m} \sim p(y)^{\otimes m},
    \label{eq:setup}
\end{equation}
with $N \ll n, m$. 
The two collections may differ in size and require no hidden pairing. Ground-truth correspondence is therefore limited to the $N$ anchors. 

No map between $\Reals^{d_x}$ and $\Reals^{d_y}$ is given. The three steps below construct one: first a common coordinate system, then a correspondence within it, then an inductive model of that correspondence. Algorithm~\ref{alg:method} presents the complete procedure.

\subsection{Anchor-relative coordinate space}
\label{sec:relative}

A point's absolute coordinates carry no meaning across modality-specific encoders. However, its \emph{similarity profile} relative to a paired anchor set provides a modality-agnostic representation on a unified scale. We therefore express all cross-modal quantities through these anchors.

\begin{definition}[Relative representation]
For a similarity $\mathrm{sim}$ on $\Reals^{d}$ and an ordered anchor set ${\mathcal{A}_Z=\{z_1, \dots, z_N\} \subset \Reals^d}$, the relative representation of $z \in \Reals^{d}$ is
\begin{equation}
    \phi_{Z}(z) = \big[\, \mathrm{sim}(z, z_1),\, \dots,\, \mathrm{sim}(z, z_N) \,\big] \in \Reals^N .
\end{equation}
\end{definition}
Applied to each modality over the anchor sets $\mathcal{A}_X = \{x_{a_k}\}_{k=1}^N$ and $\mathcal{A}_Y = \{y_{a_k}\}_{k=1}^N$:
\begin{equation}
    \phi_X(x) = \big[\, \mathrm{sim}(x, x_{a_1}),\, \dots,\, \mathrm{sim}(x, x_{a_N}) \,\big], \quad
    \phi_Y(y) = \big[\, \mathrm{sim}(y, y_{a_1}),\, \dots,\, \mathrm{sim}(y, y_{a_N}) \,\big].
\end{equation}
Since $x_{a_k}$ and $y_{a_k}$ form a multimodal anchor pair, coordinate $k$ of $\phi_X(x)$ and coordinate $k$ of $\phi_Y(y)$ measure similarity to the same concept. The two vectors are therefore directly comparable coordinate by coordinate, regardless of the dimensionality, scale, or geometry of the underlying encoders, and all subsequent cross-modal computations take place in $\Reals^N$.

\paragraph{Anchor-orthonormalized similarity.} With unit-normalized embeddings and cosine similarity, coordinate $k$ of the profile is the projection of $x$ onto anchor $k$: $\phi_X(x)=A_X x$, where $A_X\in\Reals^{N\times d_x}$ stacks the anchors as rows. Since the anchors are sampled from the data, many of them are expected to be along the same dominant directions, with the profile counting them several times. Formally, $\langle\phi_X(x),\phi_X(x')\rangle = N\,x^\top M_X x'$ with $M_X=A_X^\top A_X/N$, instead of $x^\top x'$: directions along which many anchors lie are amplified relative to the others, and, for large $N$, this squares the spectrum of the embedding covariance (Appendix~\ref{app:spectrum}). We therefore compensate for $M_X$ within each modality, to a degree set by $\gamma$:
\begin{equation}
    \mathrm{sim}_\gamma(x,x') = x^\top M_X^{-\gamma/2}\, x', \qquad \phi_X^\gamma(x) = A_X M_X^{-\gamma/2}\, x, \qquad \gamma \in [0,2].
    \label{eq:orthonorm}
\end{equation}
Coordinate $k$ remains the similarity to anchor $k$. Setting $\gamma=0$ recovers plain cosine similarity. At $\gamma=1$, when $M_X$ is invertible, $\phi_X^1/\sqrt N$ is an exact isometric embedding of $\Reals^{d_x}$: $\langle\phi_X^1(x),\phi_X^1(x')\rangle=N\,x^\top x'$ for every $x,x'$, an orthonormalization (L\"owdin) of the anchor matrix itself rather than of the embedding covariance (Appendix~\ref{app:spectrum}). At $\gamma=2$, the profile instead compensates $M_X$ twice, which whitens the embedding covariance and makes the comparison equivalent to an unregularized canonical correlation across modalities. Larger $\gamma$ relies more heavily on directions poorly estimated from only $N$ anchors. Small eigenvalues of $M_X$ are clipped for numerical stability, and $\mathcal{Y}$ is treated in the same way with its own $M_Y$. We select $\gamma$ with the procedure of Appendix~\ref{app:hyperparam} and use $\mathrm{sim}=\mathrm{sim}_\gamma$ for the remainder of this work.

Placing the two modalities in one space makes them comparable, but does not say which unpaired $y_j\in\mathcal{Y}_u$ corresponds to a given unpaired $x_i\in\mathcal{X}_u$. That assignment is the matching problem of the next section, and the obvious rule for it, nearest neighbor, as used by ASIF's retrieval rule and by RELIT's pair-mining step \citep{norelli2023asif, chen2023relit}, typically fails due to hubness.

\subsection{Global entropic optimal transport in the anchor-relative space}
\label{sec:ot}

Once both modalities are expressed in the common anchor-relative space $\Reals^N$, cross-modal alignment reduces to a matching problem between two point clouds. Let $\phi_X(x_i)$ and $\phi_Y(y_j)$ denote the anchor-relative representations of the unpaired points of Equation~\ref{eq:setup}, and define the cross-modal cost
\begin{equation}
    C_{ij} = 1-k\big(\phi_X^\gamma(x_i), \phi_Y^\gamma(y_j)\big), \qquad i \in [n],\ j \in [m],
    \label{eq:cost}
\end{equation}
where $k$ is a similarity on $\Reals^N$, such as cosine similarity. It acts on anchor-relative representations and is therefore distinct from the $\mathrm{sim}$ of Section~\ref{sec:relative}, which acts on modality-specific representations. We then solve a single entropy-regularized optimal transport problem over the entire unpaired pool,
\begin{equation}
    \pi^\star \;=\; \argmin_{\pi \in \Pi(\mu, \nu)} \; \langle \pi, C \rangle \;-\; \varepsilon H(\pi),
    \qquad
    \Pi(\mu,\nu) = \big\{ \pi \in \Reals_+^{n \times m} \;:\; \pi \mathbf{1}_m = \mu, \;\; \pi^\top \mathbf{1}_n = \nu \big\},
    \label{eq:ot}
\end{equation}
where $H(\pi) = -\sum_{ij} \pi_{ij}(\log \pi_{ij} - 1)$ is the entropy \citep{Peyré2019Computational}, $\varepsilon > 0$ the regularization strength, and $\mu, \nu$ the uniform marginals on the two sides.

\paragraph{Why a global problem.} Training-free retrieval methods compare each item against the reference collection \citep{norelli2023asif}. Since nearest-neighbor matching in high dimension is dominated by hubs, a few points are typically matched with many queries \citep{radovanovic2010hubs}. Equation~\ref{eq:ot} couples all assignments instead: its marginal constraints fix the total mass each point receives, such that every $y_j$ receives exactly $\nu_j$. Existing methods used to mitigate hubness in retrieval can be related to optimal transport. The inverted softmax \citep{smith2017offline} is the first half-step of Sinkhorn's algorithm and normalizes a single marginal \citep{pan2025hubness}. CSLS \citep{conneau2017word} instead subtracts a fixed penalty on each side, the average similarity to each point's $K$ nearest neighbors (Appendix~\ref{app:csls}). Our coupling is thus a converged solution to these methods.

\paragraph{Solving Equation~\ref{eq:ot} at collection scale.} Since the entropic solution factorizes as ${\pi^\star_{ij} = a_i \, e^{-C_{ij}/\varepsilon} \, b_j}$, it is fully determined by the dual potentials $(u, v) = (\varepsilon \log a, \varepsilon \log b)$, which occupy $O(n+m)$ memory even when $C$ does not. We thus compute $C$ block-by-block on demand and never materialize it \citep{ye2026flashsinkhorn}. The exact global coupling is therefore attainable for large pool sizes, and the problem stays rectangular, with independent marginals and no assumption that the two collections share an indexing, as Equation~\ref{eq:setup} requires.

\subsection{From a transductive coupling to inductive encoders via distillation}
\label{sec:distillation}
 
The coupling $\pi^\star$ is transductive: it is defined only over the points jointly considered when it was solved, so it cannot score an unseen point without re-solving Equation~\ref{eq:ot} against a reference collection. We therefore train two lightweight, modality-specific encoders,
\begin{equation}
    f_\theta : \Reals^{d_x} \to \Reals^{d_z}, \qquad g_\psi : \Reals^{d_y} \to \Reals^{d_z},
\end{equation}
mapping raw (not relative) embeddings into a shared $d_z$-dimensional space, so that at inference time a new modality-specific representation can be embedded with a single forward pass through $f_\theta$ or $g_\psi$, with no dependence on the anchor set or on any reference collection.
 
\paragraph{Training in chunks.}\label{sec:chunking} The encoders are trained contrastively, so each point must be scored against a set of candidates, and the pool is too large for one softmax over all of it. We therefore draw chunks $c=(c_X,c_Y)$, with $c_X \subseteq [n]$ and $c_Y \subseteq [m]$ indexed independently and $|c_X|=|c_Y|=B$, and take one gradient step per chunk. The transport itself is not chunked. Because the potentials determine $\pi^\star$ everywhere, the block a chunk needs is available in closed form,
\begin{equation}
    \pi^\star_{ij} = \exp\!\big( (u_i + v_j - C_{ij})/\varepsilon \big), \qquad i \in c_X,\ j \in c_Y,
    \label{eq:pi-block}
\end{equation}
requiring no further Sinkhorn iterations. Each chunk is therefore supervised by an exact restriction of the global plan, and the chunking is confined to the gradient estimate.

\begin{algorithm}[tb]
\caption{Representation alignment from a few paired examples via distilled OT}
\label{alg:method}
\begin{algorithmic}[1]
\Require Unlabeled pools $\mathcal{X}_u, \mathcal{Y}_u$ (drawn independently, Eq.~\ref{eq:setup}); paired anchor set $\mathcal{A}$, $|\mathcal{A}| = N$; chunk size $B$; similarity $\mathrm{sim}$; regularization $\varepsilon$
\State Compute relative representations $\phi_X(x_i)$, $\phi_Y(y_j)$ for all unpaired points, using $\mathcal{A}$ \Comment{Sec.~\ref{sec:relative}}
\State Solve the \emph{global} entropic OT problem over the full pool by streaming $C$ in tiles, obtaining the dual potentials $(u, v)$ \Comment{Sec.~\ref{sec:ot}, Eq.~\ref{eq:ot}}
\State Initialize encoders $f_\theta$, $g_\psi$
\While{not converged}
    \State Draw the next chunk $c=(c_X,c_Y)$ of size $B$, or the anchor chunk $\mathcal{A}$ with $\pi^{(\mathcal{A})}=I$ \Comment{Sec.~\ref{sec:chunking}}
    \State Recover the exact block $\pi^\star_{ij}$, $i \in c_X$, $j \in c_Y$, in closed form from $(u, v)$ \Comment{Eq.~\ref{eq:pi-block}}
    \State Compute distillation loss $\mathcal{L}^{(c)}(\theta, \psi)$ \Comment{Eq.~\ref{eq:distill-i2t}--\ref{eq:distill-t2i}}
    \State Update $\theta, \psi$ by gradient descent
\EndWhile
\State \Return inductive encoders $f_\theta$, $g_\psi$
\end{algorithmic}
\end{algorithm}

\paragraph{Soft-target distillation.} Each point of the chunk is supervised by its own row or column of the restricted plan, normalized within the chunk. Define the row- and column-normalized plans
\begin{equation}
    \bar\pi_{ij} = \frac{\pi^{\star}_{ij}}{\sum_{j' \in c_Y} \pi^{\star}_{ij'}}, \qquad
    \bar\pi'_{ij} = \frac{\pi^{\star}_{ij}}{\sum_{i' \in c_X} \pi^{\star}_{i'j}},
\end{equation}
which redistribute, for each $x_i$ (respectively each $y_j$), the transport mass it receives over the candidates present in the chunk. A point that is ambiguous or poorly covered by the anchors simply yields a more diffuse target distribution rather than a degenerate one; every point of the chunk is both a source of supervision and a candidate for every other point's softmax.

\paragraph{Direction-asymmetric distillation loss.} We minimize a soft-target contrastive loss in each direction separately, using the full chunk both as the set of supervised points and as the candidate pool against which the softmax is computed:
\begin{align}
    \mathcal{L}^{(c)}_{X\to Y}(\theta, \psi, s) &= -\frac{1}{|c_X|}\sum_{i \in c_X} \sum_{j \in c_Y} \bar\pi_{ij} \, \log \frac{\exp\big(s\, f_\theta(x_i)^\top g_\psi(y_j)\big)}{\sum_{j' \in c_Y} \exp\big(s\, f_\theta(x_i)^\top g_\psi(y_{j'})\big)}, \label{eq:distill-i2t} \\
    \mathcal{L}^{(c)}_{Y\to X}(\theta, \psi, s) &= -\frac{1}{|c_Y|}\sum_{j \in c_Y} \sum_{i \in c_X} \bar\pi'_{ij} \, \log \frac{\exp\big(s\, g_\psi(y_j)^\top f_\theta(x_i)\big)}{\sum_{i' \in c_X} \exp\big(s\, g_\psi(y_j)^\top f_\theta(x_{i'})\big)}, \label{eq:distill-t2i}
\end{align}
with $s = 1/\tau > 0$ a learned inverse-temperature scalar clamped to $s \le 100$, following CLIP \citep{radford2021clip}, and total objective $\mathcal{L}^{(c)} = \tfrac{1}{2}(\mathcal{L}^{(c)}_{X\to Y} + \mathcal{L}^{(c)}_{Y\to X})$. The anchor set enters as one further chunk with $\pi^{(\mathcal{A})} = I$, so the $N$ known pairs also supervise the encoders directly as one-hot targets, recovering a standard InfoNCE loss on this chunk. Appendix~\ref{app:distill} gives the training schedule and Appendix~\ref{app:experiments} reports further ablations.

\section{Experiments}
\label{sec:experiments}

In this section, we evaluate our method across a diverse set of experiments. We benchmark cross-modal retrieval against established baselines on two structurally different modality pairs, including image-text and molecule-text, at low anchor budget ($N\leq500$). We isolate the contribution of the marginal constraint by comparing it against one-step hubness corrections from the bilingual lexicon induction literature, and we test whether our method needs a hidden pairing in the unpaired pool.

\paragraph{Image-text.} Following the experimental setting of~\citet{yacobi2025sue}, we use Flickr30k \citep{young2014flickr30k} and MS-COCO \citep{lin2014mscoco} image-caption pairs. Image embeddings are extracted with a frozen DINOv2 ViT-S/14 backbone \citep{caron2021dino}. Caption embeddings are extracted with a frozen SBERT model, all-MiniLM-L6-v2 \citep{reimers2019sbert}.

\paragraph{Molecule-text.} We additionally use ChEBI-20 \citep{edwards2021text2mol}, a dataset of molecule-description pairs combining PubChem compounds with ChEBI textual annotations. Molecule embeddings are extracted with a frozen MoLFormer-XL \citep{ross2022molformer} encoder. Description embeddings are extracted with the same frozen SBERT model used for the image-text experiments, so that the text-side encoder is held fixed across modality pairs.

\paragraph{General protocol.} In every setting, all backbone encoders are frozen and embeddings are precomputed once. Only the lightweight encoders $f_\theta, g_\psi$ of Section~\ref{sec:distillation} are trained. The architecture of the encoders is kept identical to the ones used by SOTAlign (a single LayerNorm with a Linear projection layer). We sample $N =500$ paired anchors (Section~\ref{sec:relative}) and hold out a test split never used during OT coupling or distillation. Unless otherwise noted, we use a pool of ${\sim}29$k items split into disjoint image-only and caption-only halves, and a held-out test set of $400$ pairs, following SUE's protocol on Flickr30k. This test set is smaller than the standard Flickr30k 1K and MS-COCO 5K test sets, so our numbers are not comparable to the CLIP-era retrieval literature. We report  recall@$k$ in both retrieval directions for $k \in \{1, 5, 10\}$ on the held-out split, using nearest-neighbor search between $f_\theta(x_{\text{test}})$ and $g_\psi(y_{\text{test}})$.

\begin{table}[t!]
\caption{Cross-modal retrieval at $N=500$ anchors, recall@$k$ (\%). All baselines are re-run under our protocol, with the same frozen backbones, anchors and test split, averaged over five seeds. Our re-run of SUE exceeds its published Flickr30k results at this budget (I2T R@1 $7.80$
vs.\ $4.25$).}
\label{tab:e1}\centering
  \centering
  \footnotesize
\setlength{\tabcolsep}{3.8pt}
\begin{tabular}{lcccccccccccc}
\toprule
& \multicolumn{6}{c}{\textbf{Flickr30k}} & \multicolumn{6}{c}{\textbf{MS-COCO}} \\[1pt]
\cmidrule(lr){2-7}\cmidrule(lr){8-13}
& \multicolumn{3}{c}{\textbf{\textit{Image $\rightarrow$ Text}}} & \multicolumn{3}{c}{\textbf{\textit{Text $\rightarrow$ Image}}} & \multicolumn{3}{c}{\textbf{\textit{Image $\rightarrow$ Text}}} & \multicolumn{3}{c}{\textbf{\textit{Text $\rightarrow$ Image}}} \\
\cmidrule(lr){2-4}\cmidrule(lr){5-7}\cmidrule(lr){8-10}\cmidrule(lr){11-13}
\rowcolor{gray!6}
Method & R@1 & R@5 & R@10 & R@1 & R@5 & R@10 & R@1 & R@5 & R@10 & R@1 & R@5 & R@10 \\
\midrule
Contrastive & 22.05 & 49.10 & 60.45 & 22.30 & 48.05 & 59.60 & 14.60 & 40.65 & 54.95 & 15.00 & 40.55 & 54.30 \\
\rowcolor{gray!6}
ASIF 
& 10.35 & 29.10 & 42.45 & 9.55 & 29.80 & 43.05 & 15.15 & 48.65 & 69.95 & 14.20 & 48.80 & 69.70 \\
CSA 
& 0.90 & 3.15 & 5.60 & 0.95 & 3.55 & 6.10 & 1.30 & 5.55 & 9.40 & 1.35 & 5.70 & 8.95 \\
\rowcolor{gray!6}
SUE 
& 7.80 & 24.50 & 37.50 & 7.35 & 26.80 & 38.05 & 9.90 & 36.60 & 55.80 & 9.10 & 37.25 & 56.35 \\
SOTAlign 
& 35.60 & 65.00 & 76.20 & 35.20 & 65.10 & 75.55 & \underline{\textbf{30.50}} & 70.10 & 84.50 & \underline{\textbf{30.90}} & 69.70 & 84.50 \\
\midrule
\rowcolor{gray!6}
\textbf{Ours} & \underline{\textbf{36.55}} & \underline{\textbf{67.90}} & \underline{\textbf{79.25}} & \underline{\textbf{36.05}} &\underline{\textbf{68.10}} & \underline{\textbf{79.40}} & 29.05 & \underline{\textbf{71.05}} & \underline{\textbf{85.35}} & 30.30 & \underline{\textbf{70.95}} & \underline{\textbf{86.45}} \\
\bottomrule
\end{tabular}
\end{table}

\paragraph{Retrieval performance.} All competing methods are tested under our protocol, with the same backbones, anchors, unpaired pool (${\sim}80$k items on MS-COCO) and test split. Three use the $N$ anchors only: \textbf{ASIF} \citep{norelli2023asif}, \textbf{CSA} \citep{li2025csa}, and a \textbf{contrastive} baseline, i.e.\ our encoders and loss trained on the anchor chunk alone, which isolates the contribution of the coupling. \textbf{SUE} \citep{yacobi2025sue} and \textbf{SOTAlign} \citep{roschmann2026sotalign} also use the unpaired pool, whose image and text parts are disjoint. On image-text (Table~\ref{tab:e1}), our method is best in $10$ of $12$ columns, and SOTAlign and ours hold the top two positions in every column. Against SOTAlign, we lead by $0.9$ R@$1$ points on Flickr30k and SOTAlign leads by $1.4$/$0.6$ on MS-COCO, i.e.\ a few test queries per seed, so we consider the two tied at R@$1$. At R@$5$ and R@$10$, we lead in all eight entries ($+0.9$ to $+3.9$). Distilling the coupling multiplies the contrastive baseline's R@$1$ by $1.6$--$2.0$, a gain that comes entirely from unpaired data. ASIF, which retrieves each query independently by nearest neighbor, is $14$--$27$ R@$1$ points below ours. SUE is below the contrastive baseline at R@$1$ and R@$5$, and CSA, validated with thousands of pairs, stays near chance.

\begin{table*}[t!]
\centering
\footnotesize
\begin{minipage}[t]{0.55\textwidth}
\centering
\caption{Molecule-text retrieval on ChEBI-20 at $N=500$.}
\label{tab:e2}
\setlength{\tabcolsep}{4pt}
\begin{tabular}{lcccccc}
\toprule
& \multicolumn{3}{c}{\textbf{\textit{Molecule $\rightarrow$ Text}}} & \multicolumn{3}{c}{\textbf{\textit{Text $\rightarrow$ Molecule}}} \\
\cmidrule(lr){2-4}\cmidrule(lr){5-7}
\rowcolor{gray!6}
Method & R@1 & R@5 & R@10 & R@1 & R@5 & R@10 \\
\midrule
Contrastive & 16.10 & 37.90 & 48.20 & 15.10 & 36.15 & 47.05 \\
\rowcolor{gray!6}
ASIF     & 8.40 & 28.20 & 40.90 & 14.55 & 35.85 & 49.40 \\
CSA      & 0.20 & 1.10 & 2.55 & 0.65 & 2.45 & 4.40 \\
\rowcolor{gray!6}
SUE      & 3.25 & 14.85 & 25.05 & 3.35 & 15.10 & 26.10 \\
SOTAlign & 20.30 & 47.15 & 60.60 & 20.30 & 47.80 & 62.20 \\
\midrule
\rowcolor{gray!6}
\textbf{Ours} & \underline{\textbf{22.10}} & \underline{\textbf{49.75}} & \underline{\textbf{63.25}} & \underline{\textbf{22.05}} & \underline{\textbf{51.00}} & \underline{\textbf{63.70}} \\
\bottomrule
\end{tabular}
\end{minipage}%
\hfill
\begin{minipage}[t]{0.40\textwidth}
\centering
\caption{Zero-shot classification accuracy (\%) on $700$ ImageNet images from 7 classes. Methods trained on Flickr30k or MS-COCO. Mean (std) over 5 seeds.}
\label{tab:zeroshot}
\setlength{\tabcolsep}{4pt}
\begin{tabular}{lcc}
\toprule
Method & Flickr30k & MS-COCO \\
\midrule
\rowcolor{gray!6}
SUE      & 72.26 (2.59) & 63.97 (8.99) \\
SOTAlign & 87.51 (2.73) & \underline{\textbf{95.34}} (1.73) \\
\rowcolor{gray!6}
\textbf{Ours} & \underline{\textbf{89.91}} (3.89) & 93.63 (3.33) \\
\midrule
CLIP & \textit{98.43} & \textit{98.43} \\
\bottomrule
\end{tabular}
\end{minipage}
\end{table*}

\paragraph{Molecule-text retrieval.} On ChEBI-20 (Table~\ref{tab:e2}), where no prior method operates at $N\le500$, our method is first in all six columns, $1.5$--$3.2$ points ahead of SOTAlign. The gain over the contrastive baseline is smaller than on image-text (R@$1$ $\times1.4$--$1.5$), and SUE falls furthest, to a fifth of the contrastive baseline's R@$1$.

\begin{figure}[tb]
    \centering
    \includegraphics[width=0.99\linewidth]{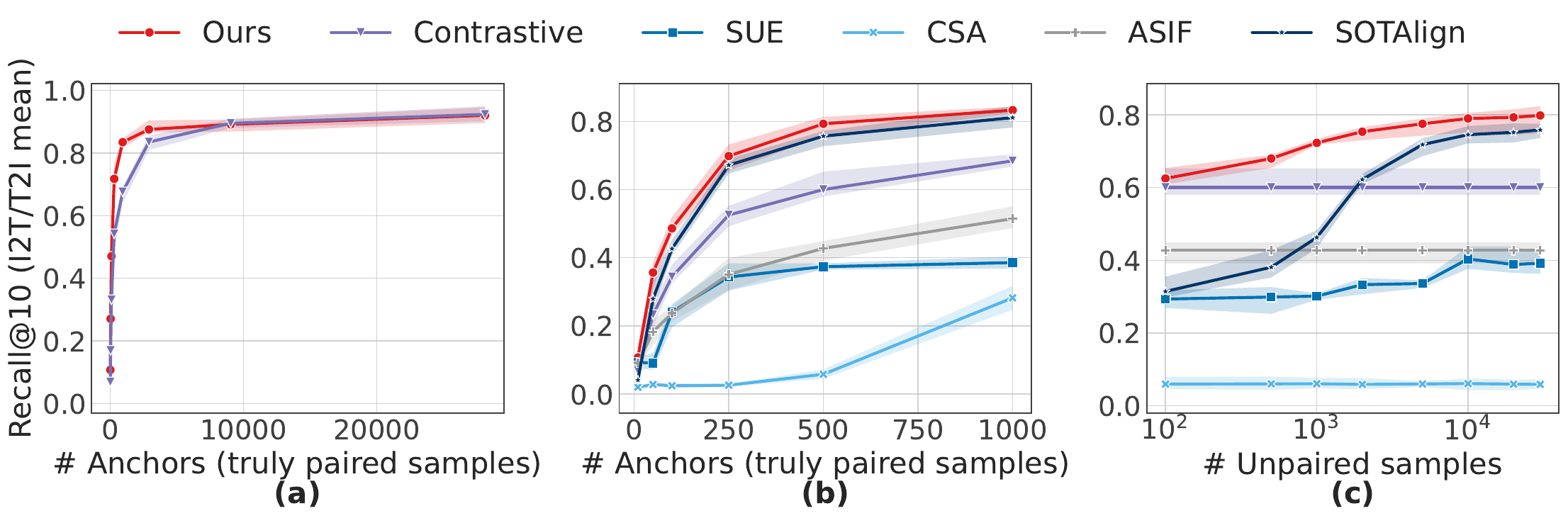}
    \caption{\textbf{Recall@10 (I2T/T2I mean) under two independent ablations on the pool composition.} \textbf{(a)--(b)} Recall@$10$ vs.\ number of anchor pairs $N$, at fixed pool and test split, over the full and low-$N$ ranges respectively: (a) restricted to our method and the contrastive baseline, (b) all six methods of Table~\ref{tab:e1}. \textbf{(c)} Recall@$10$ vs.\ unpaired pool size at fixed $N=500$ and test split.}
    \label{fig:recall_comparison}
\end{figure}
\begin{wrapfigure}{r}{0.47\textwidth}
  \centering
  \includegraphics[width=0.45\textwidth]{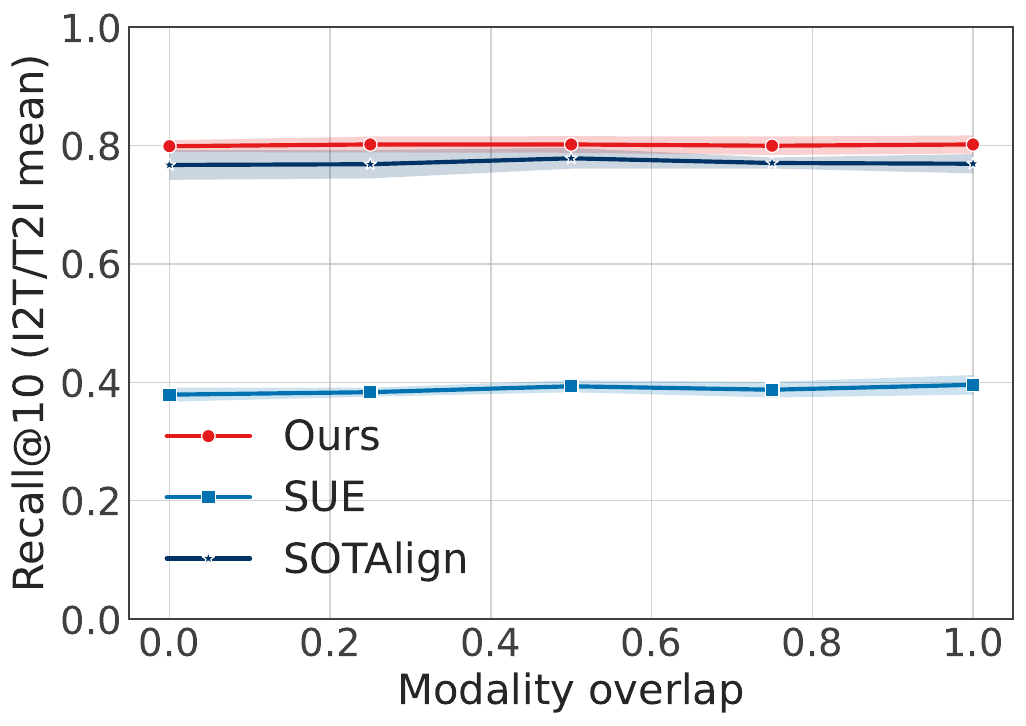}
  \caption{Recall@$10$ on Flickr30k as a function of the overlap $o$, for SUE, SOTAlign and our method.}
  \label{fig:overlap}
\end{wrapfigure}

\paragraph{Anchor budget and pool size.}
On Flickr30k (Figure~\ref{fig:recall_comparison}b), our method consistently outperforms SOTAlign and the contrastive baseline across the evaluated anchor budgets. The gap with SOTAlign remains relatively stable as the number of paired anchors increases, while both methods improve with additional paired data. When the number of anchors becomes very large (a), our method and the contrastive baseline approach similar performance from $N \sim 8k$, reaching approximately $0.9$ R@$10$. This panel isolates the contribution of our OT-learned coupling on unpaired data which decreases as the pool of unpaired is replaced by paired data. 
In contrast, increasing the size of the unpaired pool (c) consistently benefits our method, which improves from approximately $0.63$ R@$10$ with $10^2$ unpaired samples to $0.80$ with $3\times10^4$ samples.

\paragraph{Hidden pairing influence.} 
We vary the overlap $o\in[0,1]$, the fraction of the unpaired pool whose image and caption come from the same underlying pair, at fixed $N=500$, pool size and test split, and retrain all methods at every value (Figure~\ref{fig:overlap}). 
At $o=1$, every unpaired item has its counterpart on the other side, although the pairing is not revealed. At $o=0$, our default in the rest of the paper, the image-only and caption-only pools are disjoint, as when each modality is collected independently and no cross-modal pair exists beyond the $N$ anchors. Recall@$10$ is essentially unchanged across the whole range, so the methods are as effective without a hidden correspondence as with one. This is consistent with the fact that the unpaired collections can be drawn independently from the two marginals (Equation \ref{eq:setup}).

\paragraph{Hubness reduction.} 
We compare four matching rules on the same anchor-relative cost: (i) nearest neighbor, as in ASIF \citep{norelli2023asif}, (ii) the inverted softmax \citep{smith2017offline}, (iii) CSLS \citep{conneau2017word}, and (iv) our converged Sinkhorn coupling. The inverted softmax uses the same entropic kernel as our coupling but normalizes only the gallery side, so the gap between the two isolates the effect of enforcing both marginals. The rules are applied directly to the $400$ held-out test pairs, before any distillation. We measure hubness in both retrieval directions by the skewness of the $k$-occurrence distribution $N_{10}$ \citep{radovanovic2010hubs}, the number of queries whose top-$10$ contains each gallery item. On Flickr30k (Figure~\ref{fig:hubness_skewness}), our coupling has the lowest skewness in both directions ($0.32$/$0.37$, against $2.40$/$1.52$ for nearest neighbor, $1.68$/$2.18$ for the inverted softmax and $1.12$/$0.66$ for CSLS) and the highest recall@$10$ ($0.80$/$0.80$, against $0.67$/$0.64$ for nearest neighbor and $0.76$--$0.77$ for the two one-step corrections). Normalizing a single marginal is not enough: the inverted softmax even increases hubness over nearest neighbor in the text-to-image direction.

\begin{figure}[tb]
    \centering
    \includegraphics[width=\linewidth]{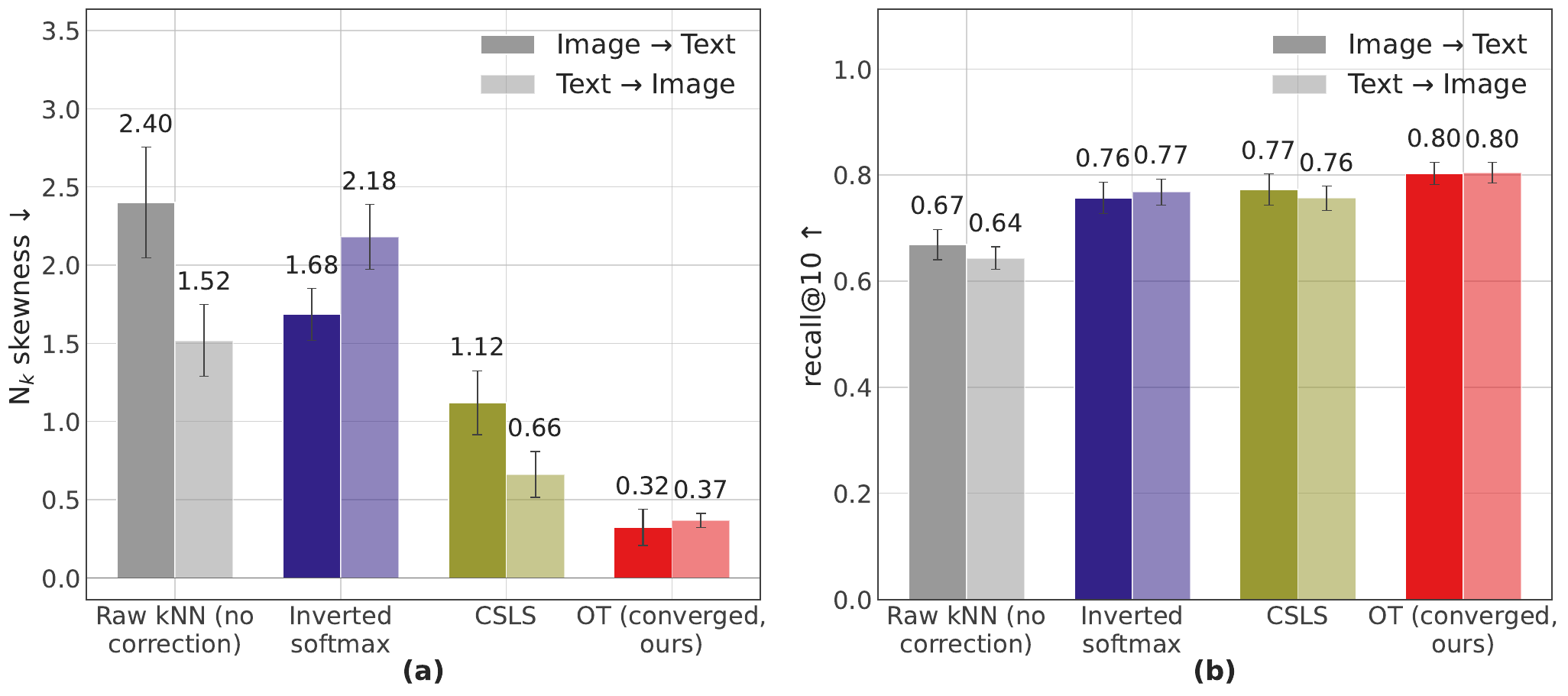}
    \caption{\textbf{(a)} $N_{10}$ skewness (lower is better) and \textbf{(b)} recall@$10$ of nearest neighbor, the inverted softmax, CSLS and our converged coupling, on the same anchor-relative cost between the $400$ held-out Flickr30k test pairs, in both retrieval directions. Error bars denote $\pm1$ std over $5$ seeds.}
    \label{fig:hubness_skewness}
\end{figure}

\paragraph{Zero-shot classification.} Following \citet{yacobi2025sue}, we classify $700$ ImageNet images from seven categories (cat, dog, man, woman, mountain, shore, snow) with the encoders trained on Flickr30k or MS-COCO, without further training, by cosine similarity to the embeddings of the prompts ``A photo of a \{\}''. Our encoders never see these images, categories or the task during training. At the same anchor budget, they outperform SUE by $18$--$30$ points and are comparable to SOTAlign, $2.4$ points ahead on Flickr30k and $1.7$ points below on MS-COCO (Table~\ref{tab:zeroshot}, examples in Appendix~\ref{sec:exp-classif}).

\paragraph{Further experiments.} 
Beyond retrieval, we analyze the evolution of the participation ratio of the learned representation with $\gamma$ (Appendix~\ref{app:pr}), compare our method against classical unsupervised approaches (Appendix~\ref{sec:exp-unsup}), and ablate two choices underlying our matching cost (Appendix~\ref{app:transport_plan}): injecting a structural term via Fused Gromov--Wasserstein and replacing the $N$-anchor relative representation itself. We then show that the learned shared representation supports qualitative cross-modal generation (Appendix~\ref{sec:exp-gen}). Finally, we evaluate zero-shot retrieval (Appendix~\ref{sec:zero_shot_retrieval}) and provide additional examples for zero-shot classification and classical retrieval (Appendices \ref{sec:exp-classif} and ~\ref{app:retrieval_ex}).

\section{Conclusion and future work}

We investigated the extreme low-resource regime of aligning two frozen unimodal encoders from only $N \sim 100$--$500$ paired examples. Our framework expresses both modalities in anchor-relative coordinates, matches the unpaired collections with a global entropic optimal transport problem whose marginal constraints reduce hubness more than one-step corrections, and distills the resulting plan into two compact, modality-specific encoders. At the same anchor budget, it outperforms anchor-relative and graph-based methods by a wide margin, matches SOTAlign at R@$1$ on image-text and outperforms it elsewhere, and requires no hidden pairing in the unpaired pool.

Zero-pair methods reach near-random recall on our encoders, and a Gromov-Wasserstein term on intra-modal geometry never improves our coupling: consistent with \citet{groger2026aristotelian}, independently trained encoders do not share enough structure to be aligned without supervision. Our method remains bounded by the quality of these encoders, and its uniform marginals assume that both collections come from the same joint distribution (Equation~\ref{eq:setup}), which unbalanced optimal transport would relax. Finally, our anchors are drawn at random, and choosing which items to pair is a promising direction for domains such as medical imaging, where each pair requires expert annotation.

\section*{Acknowledgments}
This work was supported by ANR through the ``Investissements d’avenir'' program (ANR-10-IAIHU-06, ANR-19-P3IA-0001, PRAIRIE 3IA Institute), the ``France 2030'' program (ANR-23-IACL-0008, PRAIRIE-PSAI).  The ARAMIS Lab is affiliated with DIM C-BRAINS, funded by the Conseil Régional d’Ile-de-France. This work was performed using HPC resources from GENCI-IDRIS (Grant 2026-AD011017987). RD received Marie Skłodowska-Curie grant No. 101154248 (SafeREG).

\subsection*{AI use statement}
In this work, we used generative AI tools to aid or polish writing: every paragraph of this paper was checked by AI and by us to look for writing mistakes or wrong phrasing.

We also used generative AI tools for retrieval and discovery: the pipeline was entirely designed by us without the use of AI. We performed an initial literature review and analyzed relevant papers. However, we used AI to confirm implementation details of competing methods and to search for potential missing related work. Our related work section was drafted by us.

We also used generative AI tools for research ideation and execution: we mainly used AI to test and execute our ideas. A significant part of our code was generated by AI and then manually checked to verify that it was consistent with our intent.

Finally, we used generative AI tools to draft sections of the paper: some sections were first drafted by AI (results, appendix) and then revised by us to ensure they reflected our own thinking. Tables were done manually. Figures were generated locally.

We take full responsibility for the final content of this work, including text, claims, and artifacts produced with the aid of generative AI.

\subsection*{Ethics statement}

We do not anticipate any ethical concerns arising from this work.



\bibliographystyle{plainnat}
\bibliography{References}

@article{moschella2022relative,
  title={Relative representations enable zero-shot latent space communication},
  author={Moschella, Luca and Maiorca, Valentino and Fumero, Marco and Norelli, Antonio and Locatello, Francesco and Rodol{\`a}, Emanuele},
  journal={arXiv preprint arXiv:2209.15430},
  year={2022}
}

@inproceedings{norelli2023asif,
  title={{ASIF}: Coupled Data Turns Unimodal Models to Multimodal Without Training},
  author={Norelli, Antonio and Fumero, Marco and Maiorca, Valentino and Moschella, Luca and Rodol\`a, Emanuele and Locatello, Francesco},
  booktitle={Advances in Neural Information Processing Systems (NeurIPS)},
  year={2023}
}

@inproceedings{alvarezmelis2018gromov,
  title={Gromov-Wasserstein alignment of word embedding spaces},
  author={Alvarez-Melis, David and Jaakkola, Tommi S.},
  booktitle={Conference on Empirical Methods in Natural Language Processing (EMNLP)},
  year={2018}
}

@inproceedings{grave2019unsupervised,
  title={Unsupervised alignment of embeddings with wasserstein procrustes},
  author={Grave, Edouard and Joulin, Armand and Berthet, Quentin},
  booktitle={The 22nd International Conference on Artificial Intelligence and Statistics},
  pages={1880--1890},
  year={2019},
  organization={PMLR}
}

@article{radovanovic2010hubs,
  title={Hubs in space: Popular nearest neighbors in high-dimensional data},
  author={Radovanovi{\'c}, Milo{\v{s}} and Nanopoulos, Alexandros and Ivanovi{\'c}, Mirjana},
  journal={Journal of Machine Learning Research},
  volume={11},
  pages={2487--2531},
  year={2010}
}

@article{dinu2014improving,
  title={Improving zero-shot learning by mitigating the hubness problem},
  author={Dinu, Georgiana and Lazaridou, Angeliki and Baroni, Marco},
  journal={arXiv preprint arXiv:1412.6568},
  year={2014}
}

@inproceedings{radford2021clip,
  title={Learning transferable visual models from natural language supervision},
  author={Radford, Alec and Kim, Jong Wook and Hallacy, Chris and Ramesh, Aditya and Goh, Gabriel and Agarwal, Sandhini and Sastry, Girish and Askell, Amanda and Mishkin, Pamela and Clark, Jack and others},
  booktitle={International Conference on Machine Learning (ICML)},
  year={2021}
}

@inproceedings{caron2021dino,
  title={Emerging properties in self-supervised vision transformers},
  author={Caron, Mathilde and Touvron, Hugo and Misra, Ishan and J{\'e}gou, Herv{\'e} and Mairal, Julien and Bojanowski, Piotr and Joulin, Armand},
  booktitle={International Conference on Computer Vision (ICCV)},
  year={2021}
}

@inproceedings{reimers2019sbert,
  title={Sentence-{BERT}: Sentence embeddings using {S}iamese {BERT}-networks},
  author={Reimers, Nils and Gurevych, Iryna},
  booktitle={Conference on Empirical Methods in Natural Language Processing (EMNLP)},
  year={2019}
}

@article{young2014flickr30k,
  title={From image descriptions to visual denotations: New similarity metrics for semantic inference over event descriptions},
  author={Young, Peter and Lai, Alice and Hodosh, Micah and Hockenmaier, Julia},
  journal={Transactions of the Association for Computational Linguistics},
  volume={2},
  pages={67--78},
  year={2014}
}

@inproceedings{lin2014mscoco,
  title={Microsoft {COCO}: Common objects in context},
  author={Lin, Tsung-Yi and Maire, Michael and Belongie, Serge and Hays, James and Perona, Pietro and Ramanan, Deva and Doll{\'a}r, Piotr and Zitnick, C. Lawrence},
  booktitle={European Conference on Computer Vision (ECCV)},
  year={2014}
}

@inproceedings{cannistraci2023bootstrapping,
  title={Bootstrapping Parallel Anchors for Relative Representations},
  author={Cannistraci, Irene and Moschella, Luca and Maiorca, Valentino and Fumero, Marco and Norelli, Antonio and Rodol\`a, Emanuele},
  booktitle={Tiny Papers Track, International Conference on Learning Representations (ICLR)},
  year={2023}
}

@inproceedings{chen2023relit,
  title={Weakly Supervised Vision-and-Language Pre-training with Relative Representations},
  author={Chen, Chi and Li, Peng and Sun, Maosong and Liu, Yang},
  booktitle={Proceedings of the 61st Annual Meeting of the Association for Computational Linguistics (ACL)},
  year={2023}
}

@inproceedings{yacobi2025sue,
  title={Learning Shared Representations from Unpaired Data},
  author={Yacobi, Amitai and Ben-Ari, Nir and Talmon, Ronen and Shaham, Uri},
  booktitle={Advances in Neural Information Processing Systems (NeurIPS)},
  year={2025}
}

@article{conneau2017word,
  title={Word translation without parallel data},
  author={Conneau, Alexis and Lample, Guillaume and Ranzato, Marc'Aurelio and Denoyer, Ludovic and J{\'e}gou, Herv{\'e}},
  journal={arXiv preprint arXiv:1710.04087},
  year={2017}
}

@inproceedings{vayer2019fgw,
  title={Optimal Transport for structured data with application on graphs},
  author={Vayer, Titouan and Chapel, Laetitia and Flamary, R{\'e}mi and Tavenard, Romain and Courty, Nicolas},
  booktitle={International Conference on Machine Learning (ICML)},
  year={2019}
}

@article{gholamzadeh2025model,
  title={Model alignment using inter-modal bridges},
  author={Gholamzadeh, Ali and Sajid, Noor},
  journal={arXiv preprint arXiv:2505.12322},
  year={2025}
}

@article{ross2022molformer,
  title={Large-scale chemical language representations capture molecular structure and properties},
  author={Ross, Jerret and Belgodere, Brian and Chenthamarakshan, Vijil and Padhi, Inkit and Mroueh, Youssef and Das, Payel},
  journal={Nature Machine Intelligence},
  volume={4},
  pages={1256--1264},
  year={2022}
}

@inproceedings{edwards2021text2mol,
  title={Text2Mol: Cross-Modal Molecule Retrieval with Natural Language Queries},
  author={Edwards, Carl and Zhai, ChengXiang and Ji, Heng},
  booktitle={Conference on Empirical Methods in Natural Language Processing (EMNLP)},
  year={2021}
}

@article{smith2017offline,
  title={Offline bilingual word vectors, orthogonal transformations and the inverted softmax},
  author={Smith, Samuel L and Turban, David HP and Hamblin, Steven and Hammerla, Nils Y},
  journal={arXiv preprint arXiv:1702.03859},
  year={2017}
}

@article{Peyré2019Computational,
  title={Computational Optimal Transport},
  author={Gabriel Peyré and Marco Cuturi},
  year={2019},
  journal={now publishers, Inc. eBooks},
  doi={10.1561/9781680835519}
}

@misc{ye2026flashsinkhorn,
  title={{FlashSinkhorn}: {IO}-Aware Entropic Optimal Transport on {GPU}},
  author={Ye, Felix X.-F. and Li, Xingjie and Yu, An and Chang, Ming-Ching and Chu, Linsong and Wertheimer, Davis},
  year={2026},
  eprint={2602.03067},
  archivePrefix={arXiv},
  primaryClass={cs.LG}
}

@article{wu2021data,
  title={Data efficient language-supervised zero-shot recognition with optimal transport distillation},
  author={Wu, Bichen and Cheng, Ruizhe and Zhang, Peizhao and Gao, Tianren and Vajda, Peter and Gonzalez, Joseph E},
  journal={arXiv preprint arXiv:2112.09445},
  year={2021}
}

@inproceedings{caron2020swav,
  title={Unsupervised Learning of Visual Features by Contrasting Cluster Assignments},
  author={Caron, Mathilde and Misra, Ishan and Mairal, Julien and Goyal, Priya and Bojanowski, Piotr and Joulin, Armand},
  booktitle={Advances in Neural Information Processing Systems (NeurIPS)},
  year={2020}
}

@inproceedings{bogolin2022qbnorm,
  title={Cross Modal Retrieval with Querybank Normalisation},
  author={Bogolin, Simion-Vlad and Croitoru, Ioana and Jin, Hailin and Liu, Yang and Albanie, Samuel},
  booktitle={IEEE/CVF Conference on Computer Vision and Pattern Recognition (CVPR)},
  year={2022}
}

@misc{groger2025structure,
  title={With Limited Data for Multimodal Alignment, Let the {STRUCTURE} Guide You},
  author={Gr{\"o}ger, Fabian and Wen, Shuo and Le, Huyen and Brbi{\'c}, Maria},
  year={2025},
  eprint={2506.16895},
  archivePrefix={arXiv},
  primaryClass={cs.LG}
}

@misc{kim2026relative,
  title={Learning Relative Representations for Fine-Grained Multimodal Alignment with Limited Data},
  author={Kim, Shiwon and Park, Yu Rang},
  year={2026},
  eprint={2605.16834},
  archivePrefix={arXiv},
  primaryClass={cs.CV}
}

@article{huh2024platonic,
  title={The platonic representation hypothesis},
  author={Huh, Minyoung and Cheung, Brian and Wang, Tongzhou and Isola, Phillip},
  journal={arXiv preprint arXiv:2405.07987},
  year={2024}
}

@inproceedings{groger2026aristotelian,
  title={Revisiting the Platonic Representation Hypothesis: An Aristotelian View},
  author={Gr{\"o}ger, Fabian and Wen, Shuo and Brbi{\'c}, Maria},
  booktitle={International Conference on Machine Learning (ICML)},
  year={2026}
}

@inproceedings{li2025csa,
  title={Csa: Data-efficient mapping of unimodal features to multimodal features},
  author={Li, Po-han and Chinchali, Sandeep and others},
  booktitle={International Conference on Learning Representations},
  volume={2025},
  pages={99459--99478},
  year={2025}
}

@inproceedings{pan2025hubness,
  title={Hubness Reduction with Dual Bank Sinkhorn Normalization for Cross-Modal Retrieval},
  author={Pan, Zhengxin and others},
  booktitle={Proceedings of the 33rd ACM International Conference on Multimedia},
  year={2025}
}

@article{roschmann2026sotalign,
  title={SOTAlign: Semi-supervised alignment of unimodal vision and language models via optimal transport},
  author={Roschmann, Simon and Krzakala, Paul and Mazelet, Sonia and Bouniot, Quentin and Akata, Zeynep},
  journal={arXiv preprint arXiv:2602.23353},
  year={2026}
}

@article{gu2022keypoint,
  title={Keypoint-guided optimal transport with applications in heterogeneous domain adaptation},
  author={Gu, Xiang and Yang, Yucheng and Zeng, Wei and Sun, Jian and Xu, Zongben},
  journal={Advances in Neural Information Processing Systems},
  volume={35},
  pages={14972--14985},
  year={2022}
}

@article{tov2021designing,
  title={Designing an encoder for stylegan image manipulation},
  author={Tov, Omer and Alaluf, Yuval and Nitzan, Yotam and Patashnik, Or and Cohen-Or, Daniel},
  journal={ACM Transactions on graphics (TOG)},
  volume={40},
  number={4},
  pages={1--14},
  year={2021},
  publisher={ACM New York, NY, USA}
}

@inproceedings{mo2023sclip,
  title={{S-CLIP}: Semi-supervised Vision-Language Learning using Few Specialist Captions},
  author={Mo, Sangwoo and Kim, Minkyu and Lee, Kyungmin and Shin, Jinwoo},
  booktitle={Advances in Neural Information Processing Systems (NeurIPS)},
  year={2023}
}

\appendix

\section{Distillation details}
\label{app:distill}

\paragraph{Anchors as a distinguished chunk.}
The anchor pairs $\mathcal{A}$ are appended as an additional, distinguished chunk with plan $\pi^{(\mathcal{A})} = I$ (their correspondence being exactly known rather than inferred by transport).

\paragraph{Properties of the distillation objective.}
$\mathcal{L}^{(c)}_{X\to Y}$ and $\mathcal{L}^{(c)}_{Y\to X}$ (Equations~\ref{eq:distill-i2t}-\ref{eq:distill-t2i}) form a direction-asymmetric soft-target loss: the two directions use row- and column-normalized targets $\bar\pi$, $\bar\pi'$ respectively, which differ unless $\pi^{\star}$ is symmetric. $\mathcal{L}^{(c)}_{X\to Y}$ reduces to a standard InfoNCE loss with hard, one-hot targets when $\bar\pi_{i,:}$ is one-hot for every $i \in c_X$, and more generally distills the residual uncertainty the transport plan assigns to ambiguous matches (e.g. two visually similar images with semantically close captions). Each chunk (including the anchor chunk) is visited exactly once per epoch, in an order reshuffled every epoch, and contributes a single stochastic-gradient step from its own $\mathcal{L}^{(c)}$. The plan $\pi^{\star}$ is computed once before optimization begins and held fixed throughout training.

Because $f_\theta$ and $g_\psi$ operate directly on raw embeddings rather than on relative representations, the trained encoders no longer require access to the anchor set $\mathcal{A}$, or to any reference collection, at inference time. This is the precise sense in which the distillation step removes the transductive constraint introduced in Section~\ref{sec:ot}: the anchors are used only to \emph{supervise} training, not to compute the final representation of a query.

 \begin{figure}[h!]
    \centering
    \includegraphics[width=0.6\linewidth]{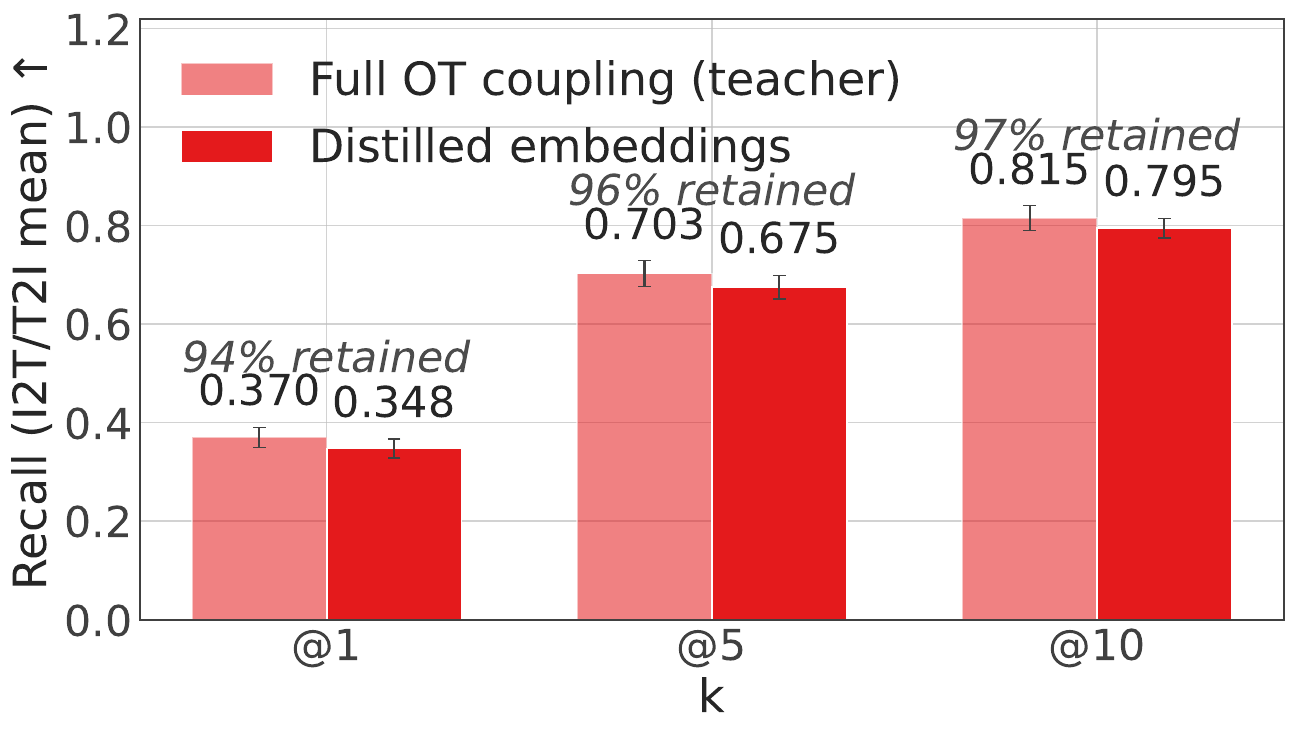}
    \caption{Recall comparison of the transductive OT teacher with how much is retained after distillation to the inductive encoders}
    \label{fig:distilledOT}
\end{figure}

\section{Theory details}
\label{app:theory}

\subsection{Spectral distortion of anchor-relative profiles}
\label{app:spectrum}
We clarify the spectrum-squaring effect referenced in Section~\ref{sec:relative} and its correction by anchor orthonormalization (Eq.~\ref{eq:orthonorm}).

\paragraph{Setup.} Let $x\sim p(x)$ with second-moment matrix $\Sigma_X=\mathbb E[xx^\top]=US U^\top$, $S=\mathrm{diag}(s_1,\dots,s_{d_x})$, assumed positive definite, so that $s_i>0$; for unit-normalized embeddings, $\sum_i s_i=\mathrm{tr}\,\Sigma_X=1$. As in Section~\ref{sec:relative}, the anchors $x_{a_1},\dots,x_{a_N}\stackrel{\mathrm{iid}}\sim p(x)$ are stacked as rows of $A_X\in\Reals^{N\times d_x}$, and $M_X= A_X^\top A_X/N$. Then $\mathbb E[M_X]=\Sigma_X$ and, by the strong law of large numbers, $M_X\to\Sigma_X$ almost surely as $N\to\infty$. Since $\Sigma_X$ is positive definite, almost surely $M_X$ is invertible for all sufficiently large $N$ (for every $N\ge d_x$ if $p$ has a density on the unit sphere).

\paragraph{Eigenvalue clipping.} In practice the eigenvalues of $M_X$ are clipped below at $\tau=10^{-6}$ times their mean, i.e.\ $\tau=10^{-6}/d_x$ for unit-normalized anchors, before taking powers. Write $\max(M_X,\tau I)$ for $M_X$ with its eigenvalues clipped at $\tau$. It is a spectral function of $M_X$ and therefore commutes with $M_X$, so the proof below goes through verbatim with $M_X^{1-\gamma}$ replaced by $M_X\max(M_X,\tau I)^{-\gamma}$. The two coincide whenever $\lambda_{\min}(M_X)\ge\tau$, which holds almost surely for large $N$ if $s_{\min}>\tau$. We therefore state the results without clipping.

\paragraph{Anchor orthonormalization at $\gamma=1$.} $A_X M_X^{-1/2}=\sqrt{N}\,A_X(A_X^\top A_X)^{-1/2}$ is $\sqrt{N}$ times the polar factor of $A_X$, its symmetric (L\"owdin) orthonormalization: its columns are orthonormal and, when $N>d_x$, its rows (the transformed anchors) form a tight frame of $\Reals^{d_x}$ rather than an orthonormal set. Consequently $\langle\phi_X^1(x),\phi_X^1(x')\rangle = N\,x^\top x'$ for all $x,x'$: $\phi_X^1/\sqrt{N}$ is an isometry of $\Reals^{d_x}$ into $\Reals^N$, exactly at every $N$ with $M_X$ invertible.

\begin{proposition}[Spectrum interpolation]
\label{prop:spectrum}
For $\gamma\in[0,2]$, let $\phi_X^\gamma(x) = A_X M_X^{-\gamma/2} x \in \Reals^N$ denote the profile induced by $\mathrm{sim}_\gamma$ of Eq.~\ref{eq:orthonorm}, and $\Gamma_X^\gamma=\mathbb E_x\big[\phi_X^\gamma(x)\phi_X^\gamma(x)^\top\big]$ its second-moment matrix at fixed anchors.
(i) For any anchors with $M_X$ invertible, $\Gamma_X^\gamma$ has exactly $d_x$ nonzero eigenvalues, counted with multiplicity, equal to those of $N\,\Sigma_X^{1/2}M_X^{1-\gamma}\Sigma_X^{1/2}$. In particular, for $\gamma=1$ they are exactly $\{N s_i\}_{i=1}^{d_x}$.
(ii) With iid anchors as above, as $N\to\infty$, the sorted nonzero eigenvalues of $\Gamma_X^\gamma/N$ converge almost surely to the sorted $\{s_i^{2-\gamma}\}_{i=1}^{d_x}$.
\end{proposition}
\begin{proof}
(i) By definition of $\Gamma_X^\gamma$, the anchors are fixed and the expectation is over $x$ only, so
\[
\Gamma_X^\gamma = A_X M_X^{-\gamma/2}\,\mathbb E[xx^\top]\,M_X^{-\gamma/2} A_X^\top = PP^\top, \qquad P = A_X M_X^{-\gamma/2}\Sigma_X^{1/2}\in\Reals^{N\times d_x}.
\]
The matrices $PP^\top\in\Reals^{N\times N}$ and $P^\top P\in\Reals^{d_x\times d_x}$ share their nonzero eigenvalues. Using $A_X^\top A_X = N M_X$ and the fact that $M_X^{-\gamma/2}$ commutes with $M_X$,
\[
P^\top P = \Sigma_X^{1/2} M_X^{-\gamma/2}(A_X^\top A_X) M_X^{-\gamma/2}\Sigma_X^{1/2} = N\,\Sigma_X^{1/2} M_X^{1-\gamma} \Sigma_X^{1/2},
\]
which is positive definite as a congruence of the positive definite $M_X^{1-\gamma}$ by the invertible $\Sigma_X^{1/2}$, so all its $d_x$ eigenvalues are nonzero. This proves the identity. For $\gamma=1$, $M_X^{0}=I$ and $P^\top P = N\Sigma_X$, whose eigenvalues are $\{N s_i\}$.

(ii) Almost surely, $M_X$ is positive definite for all large $N$ and $M_X\to\Sigma_X$. The map $M\mapsto M^{1-\gamma}$ is continuous on positive definite matrices for any real exponent, so $\Sigma_X^{1/2}M_X^{1-\gamma}\Sigma_X^{1/2}\to\Sigma_X^{2-\gamma}$ almost surely. By Weyl's inequality, $|\lambda_i(B)-\lambda_i(B')|\le\|B-B'\|_2$ for the sorted eigenvalues of symmetric matrices $B,B'$. Hence the sorted eigenvalues of $\Gamma_X^\gamma/N$ converge to those of $\Sigma_X^{2-\gamma}=US^{2-\gamma}U^\top$, namely the sorted $\{s_i^{2-\gamma}\}_{i=1}^{d_x}$.
\end{proof}
Proposition~\ref{prop:spectrum} recovers the reference points used above: $\gamma=0$ squares the spectrum ($s_i^2$), matching plain cosine similarity. $\gamma=1$ preserves its shape ($s_i^1=s_i$), exactly at every $N$, the exact isometry described above. $\gamma=2$ flattens it entirely in the limit ($s_i^0=1$), equivalent to fully whitening $\Sigma_X$. Only the $\gamma=1$ case is exact at every finite $N$ (Statement (i)). For $\gamma\ne 1$, Statement (ii) is an asymptotic approximation, accurate for the leading eigenvalues once $N$ is large relative to $d_x$ 

\begin{remark}[Effective dimensionality]
\label{rem:pr}
The participation ratio $\mathrm{PR}(\{\lambda_i\})=(\sum_i\lambda_i)^2/\sum_i\lambda_i^2$ of the limiting spectrum $\{s_i^{2-\gamma}\}$ increases monotonically with $\gamma\in[0,2]$.
\end{remark}

\subsection{Relation to CSLS.}
\label{app:csls}

CSLS \citep{conneau2017word}, another one-step hubness correction used in retrieval, shares the
additive structure of one full Sinkhorn iteration. Writing $u_i = \varepsilon \log a_i$ and
$v_j = \varepsilon \log b_j$, taking $\varepsilon \log(\cdot)$ of the factorized coupling $\pi_{ij}
= a_i K_{ij} b_j$ with $K_{ij} = \exp(\mathrm{sim}_{ij}/\varepsilon)$ gives, at any iteration $t$,
\begin{equation}
  \varepsilon \log \pi^{(t)}_{ij} = \mathrm{sim}_{ij} + v_j^{(t)} + u_i^{(t)}.
  \label{eq:sinkhorn-additive}
\end{equation}
This matches the shape of $\mathrm{CSLS}(x_i, y_j) = 2\cos(x_i, y_j) - r_T(x_i) - r_S(y_j)$: a raw
similarity corrected by one additive term per side, $r_S$ and $r_T$ being the average similarity to
the $K$ nearest neighbors in the opposite space. We make this resemblance precise by writing out one full Sinkhorn iteration both potentials updated in parallel from $a^{(0)} \equiv b^{(0)} \equiv 1$ (absorbing the constant $\log \mu_i, \log \nu_j$ terms):
\begin{equation}
  v_j^{(1)} = -\varepsilon \,\mathrm{LSE}_i\!\left(\mathrm{sim}_{ij}/\varepsilon\right), \qquad
  u_i^{(1)} = -\varepsilon \,\mathrm{LSE}_j\!\left(\mathrm{sim}_{ij}/\varepsilon\right),
  \label{eq:potentials-lse}
\end{equation}
with $\mathrm{LSE}_i(z_i) = \log \sum_i \exp(z_i)$. Substituting into
Equation~\ref{eq:sinkhorn-additive} at $t=1$ gives
\[
  \varepsilon \log \pi^{(1)}_{ij} = \mathrm{sim}_{ij}
    - \varepsilon \,\mathrm{LSE}_j(\mathrm{sim}_{ij}/\varepsilon)
    - \varepsilon \,\mathrm{LSE}_i(\mathrm{sim}_{ij}/\varepsilon),
\]

Fix $j$, write $z_i = \mathrm{sim}_{ij}$ and let $M_j = \max_i z_i$. Factoring $\exp(M_j/\varepsilon)$
out of the sum inside the $\mathrm{LSE}$,
\[
  \varepsilon \,\mathrm{LSE}_i(z_i/\varepsilon) = M_j + \varepsilon \log S_j(\varepsilon), \qquad
  S_j(\varepsilon) = \sum_i \exp\!\left(\frac{z_i - M_j}{\varepsilon}\right) \in [1,n].
\]
Since $\log S_j(\varepsilon)$ is bounded uniformly in $\varepsilon$, $\varepsilon\log S_j(\varepsilon) \to
0$ and $\varepsilon\,\mathrm{LSE}_i(z_i/\varepsilon) \to M_j$ as $\varepsilon\to 0$. The same argument
applies symmetrically to $u_i^{(1)}$, giving $v_j^{(1)} \to -\max_i \mathrm{sim}_{ij}$ and $u_i^{(1)}
\to -\max_j \mathrm{sim}_{ij}$. Substituting back into Equation~\ref{eq:sinkhorn-additive} at $t=1$,

\[
  \varepsilon \log \pi^{(1)}_{ij} = \underbrace{\mathrm{sim}_{ij}}_{\cos(x_i,y_j)}
  \underbrace{- \text{max}_{j'} \mathrm{sim}_{ij'} - \text{max}_{i'} \mathrm{sim}_{i'j}}_{-r_T(x_i) -r_S(y_j)}
  - \underbrace{\big(\varepsilon\log S_i(\varepsilon) + \varepsilon\log S_j(\varepsilon)\big)}_{\to~0 \text{ as } \varepsilon\to 0}
\]
exactly the CSLS shape with coefficient $1$ instead of $2$, $K=1$ and an additive
$f(\varepsilon) = -\varepsilon\log S_i(\varepsilon) - \varepsilon\log S_j(\varepsilon)$ function that vanishes as $\varepsilon \to 0$.

\section{Additional experiments}
\label{app:experiments}

\subsection{Participation ratio}
\label{app:pr}

\begin{figure}[h!]
    \centering
    \begin{subfigure}[t]{0.32\linewidth}
        \centering
        \includegraphics[width=\linewidth]{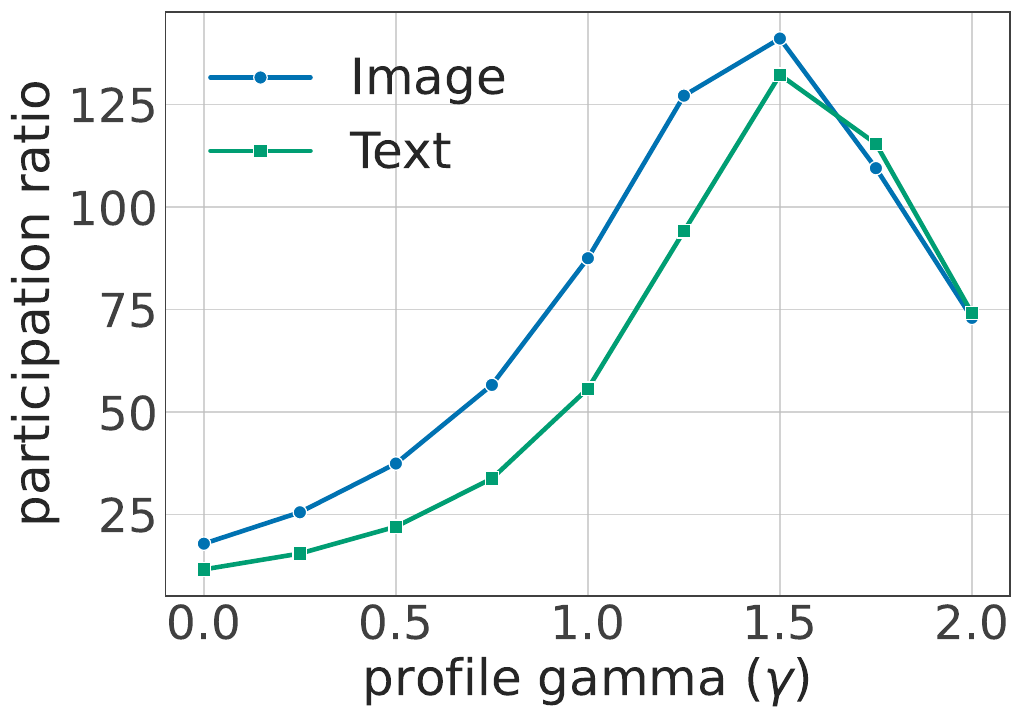}
        \caption{$N=500$}
        \label{fig:prN500}
    \end{subfigure}
    \hfill
    \begin{subfigure}[t]{0.32\linewidth}
        \centering
        \includegraphics[width=\linewidth]{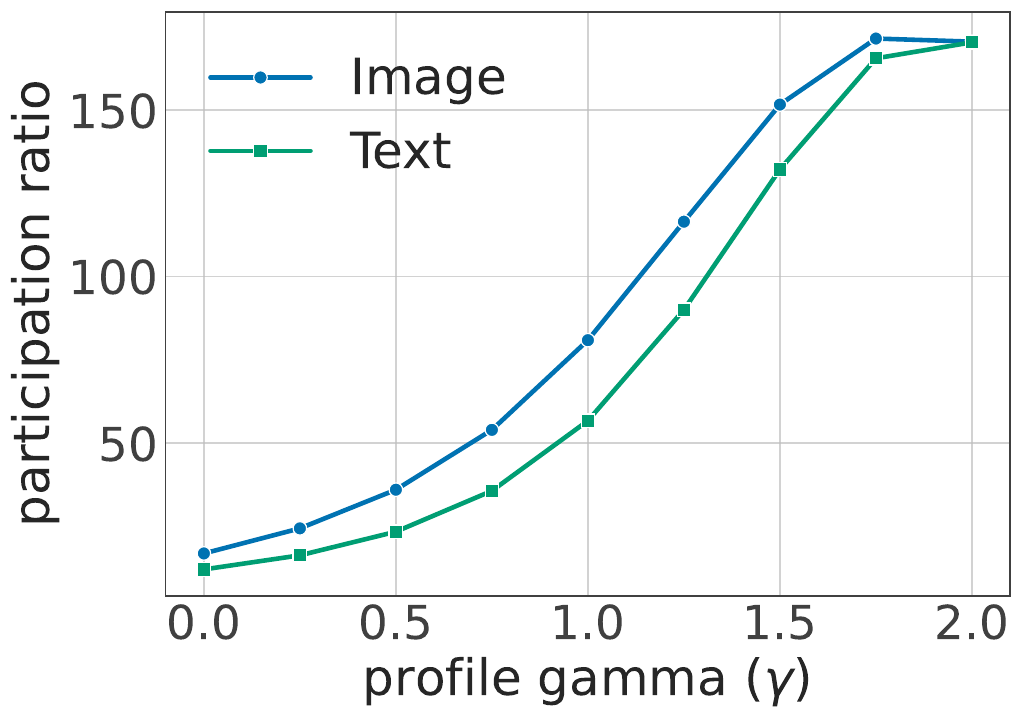}
        \caption{$N=2000$}
        \label{fig:prN2000}
    \end{subfigure}
    \hfill
    \begin{subfigure}[t]{0.32\linewidth}
        \centering
        \includegraphics[width=\linewidth]{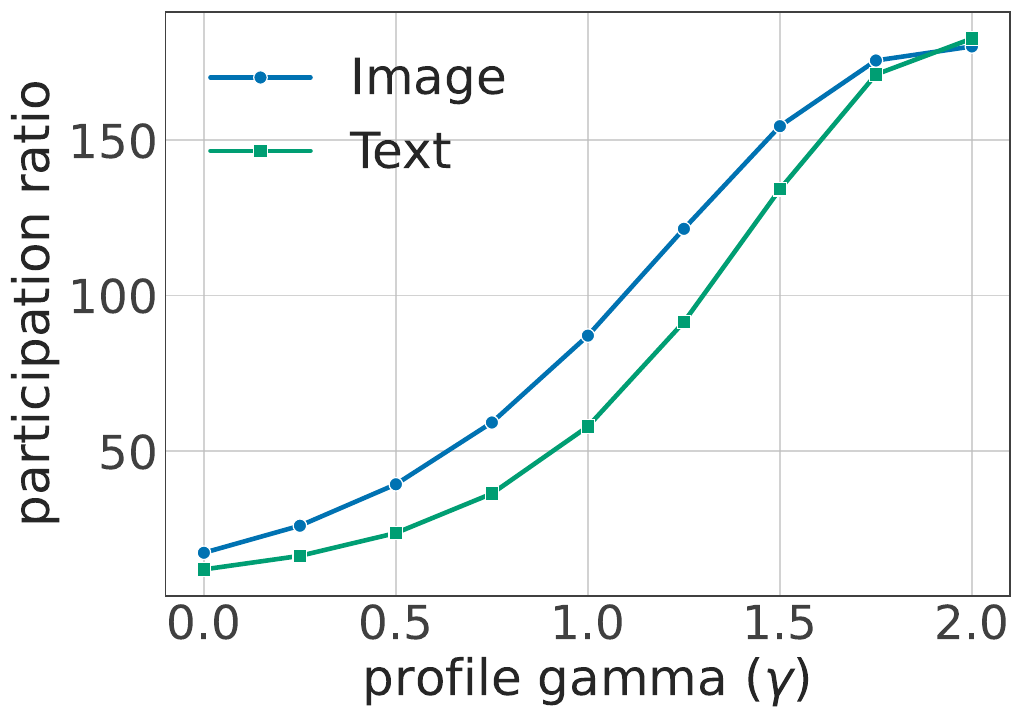}
        \caption{$N=4000$}
        \label{fig:prN4000}
    \end{subfigure}
    \caption{Participation ratio as a function of $\gamma$ evaluated on Flickr30k for three anchor budgets $N$.}
    \label{fig:pr}
\end{figure}

 Figure~\ref{fig:pr} reports the participation ratio of the anchor-relative profile on Flickr30k, for both modalities, as $\gamma$ is varied from $0$ to $2$, at three anchor budgets $N\in\{500,2000,4000\}$. At $N=2000$ and $N=4000$, $\mathrm{PR}$ increases monotonically with $\gamma$ for both modalities, as predicted by the Remark \ref{rem:pr}, saturating near its asymptotic ceiling $d_x$ as $\gamma\to2$. At $N=500$, however, the curve peaks around $\gamma\approx1.5$ and drops sharply toward $\gamma=2$. This is consistent with a finite-$N$ estimation bias rather than a failure of the underlying result: the negative exponent on $M_X$ at $\gamma>1$ disproportionately amplifies its smallest eigenvalues, which are also the ones $N$ anchors estimate least reliably (Appendix~\ref{app:spectrum}). At $N=500$ against $d_x\approx384$, a handful of poorly-estimated directions can dominate the empirical spectrum once inverted, depressing $\mathrm{PR}$ well below its asymptotic value. As $N$ grows, $M_X$ estimates $\Sigma_X$ more accurately in exactly these tail directions, and the curve converges to the monotonic prediction of Proposition~\ref{prop:spectrum}. This is precisely why we restrict $\gamma$ to $[0,1]$ in practice (Appendix~\ref{app:hyperparam}): at our target budget of a few hundred anchors, $\gamma>1$ risks amplifying anchor-estimation noise rather than correcting a real bias, well before reaching its asymptotic behavior.

\subsection{Comparison to unsupervised methods}
\label{sec:exp-unsup}

 Figure~\ref{fig:unsup} evaluates the zero-pair setting: Gromov-Wasserstein, Wasserstein Procrustes and Iterative Closest Point align the two point clouds using only their intra-modal geometry, with no anchors at all. All three collapse to near-random recall. This gap is not a failure of the matching algorithms themselves, which are well established. However, it indicates that the intra-modal geometries of our frozen image and text encoders are not yet close enough to be matched on structure alone. Unlike two instances of the same modality (e.g.\ word embeddings of different languages), where such zero-pair methods are known to succeed, image and text encoders trained independently offer not enough shared geometric signal for these methods to exploit. A handful of paired anchors is therefore a necessity in this regime.
 
 \begin{figure}[h!]
    \centering
    \includegraphics[width=0.99\linewidth]{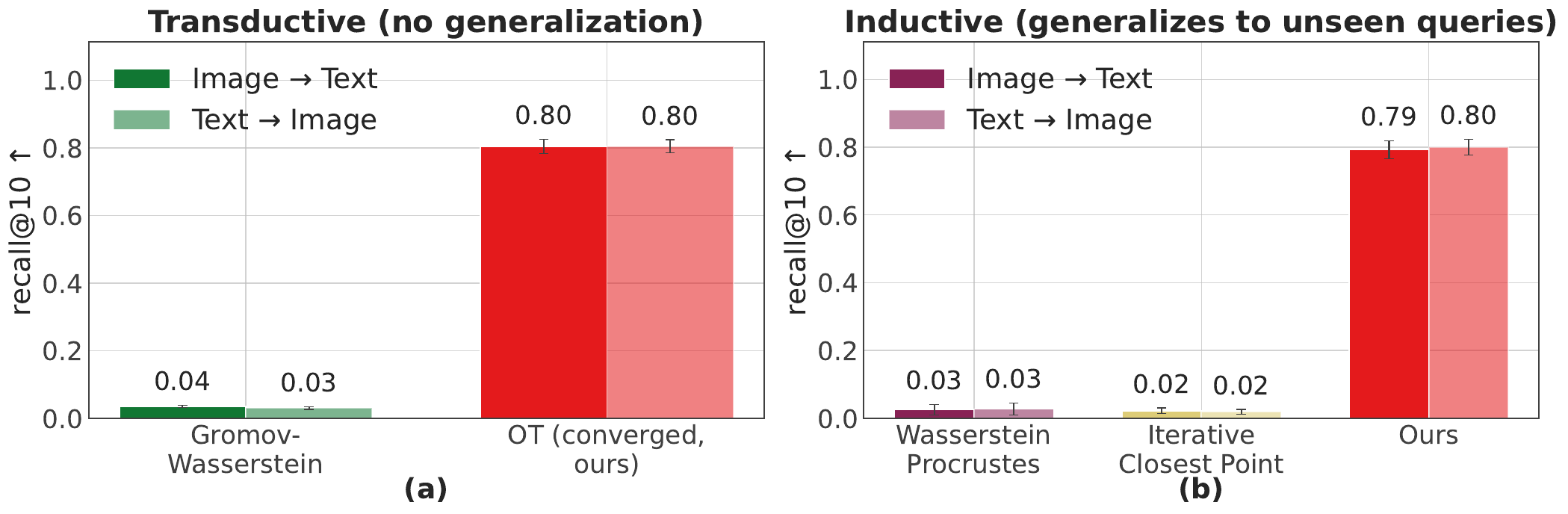}
    \caption{Recall comparison of our method to existing unsupervised baselines on Flickr30k with the same usual setting}
    \label{fig:unsup}
\end{figure}

\subsection{Choice of the optimal transport plan}
\label{app:transport_plan}

We now focus on the cost used to build our coupling, and on how
effectively it exploits the information available beyond the $N$
anchors. Sections~\ref{sec:exp-BridgeCost} and~\ref{sec:KPG-RL} respectively compare our anchor-relative representation against the bridge cost of \citet{gholamzadeh2025model}, an alternative construction that routes cross-modal distance through the anchors, and the cost of \citet{gu2022keypoint}, another construction that represents each point by a softmax-normalized distribution over its distances to the anchors, compared across modalities via Jensen-Shannon divergence. Then the section~\ref{sec:exp-FGW} studies the effect of adding a structural constraint on top of our coupling through a Fused Gromov-Wasserstein \citep{vayer2019fgw}
formulation, which interpolates between our anchor-relative cost and a
Gromov-Wasserstein penalty on the intra-modal geometry.

\subsubsection{Bridge-cost}
\label{sec:exp-BridgeCost}

\begin{figure}[h]
    \centering
    \includegraphics[width=0.99\linewidth]{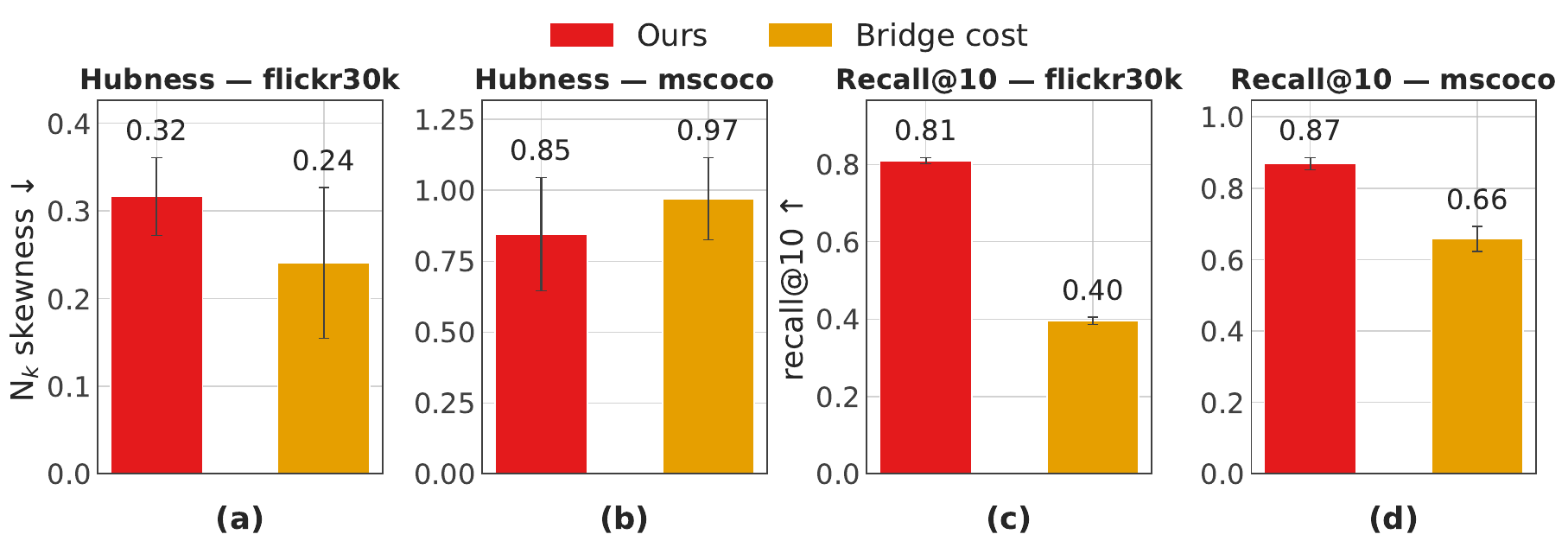}
    \caption{Hubness and recall@10 comparison on Flickr30k and MS-COCO for the Bridge cost against our method. The $N_{10}$ skewness gives similar results while the recall@10 drops considerably when using the Bridge cost suggesting that our method benefits from representing a point from its similarity to all anchors and not only the closest one.}
    \label{fig:bridge-cost}
\end{figure}

As an alternative to our anchor-relative profile, we consider the
bridge cost of \citet{gholamzadeh2025model}, which also builds a
cross-modal transport cost solely from distances to the $N$ anchors, but
combines them differently. For a query $x$ in modality $X$ and $y$ in
modality $Y$, the bridge cost is
\[
C_{\text{bridge}}(x, y) = \min_{i \in \{1, \dots, N\}} \; d_X(x, x_{a_i}) + d_Y(y_{a_i}, y),
\]
the length of the shortest two-hop path from $x$ to $y$ routed through a
single anchor $x_{a_i}$. This differs from our profile representation, which
embeds every point into the full $N$-dimensional coordinate system given
by its distance to \emph{every} anchor and compares two such coordinate
vectors directly (e.g.\ by cosine distance), rather than selecting a
single best anchor per pair. We plug this cost into the same Sinkhorn
coupling used throughout the paper, changing nothing else, and compare the resulting retrieval performance and hubness to our own cost in Figure~\ref{fig:bridge-cost}.

\subsubsection{KPG-RL cost}
\label{sec:KPG-RL}

\begin{figure}[h]
    \centering
    \includegraphics[width=0.99\linewidth]{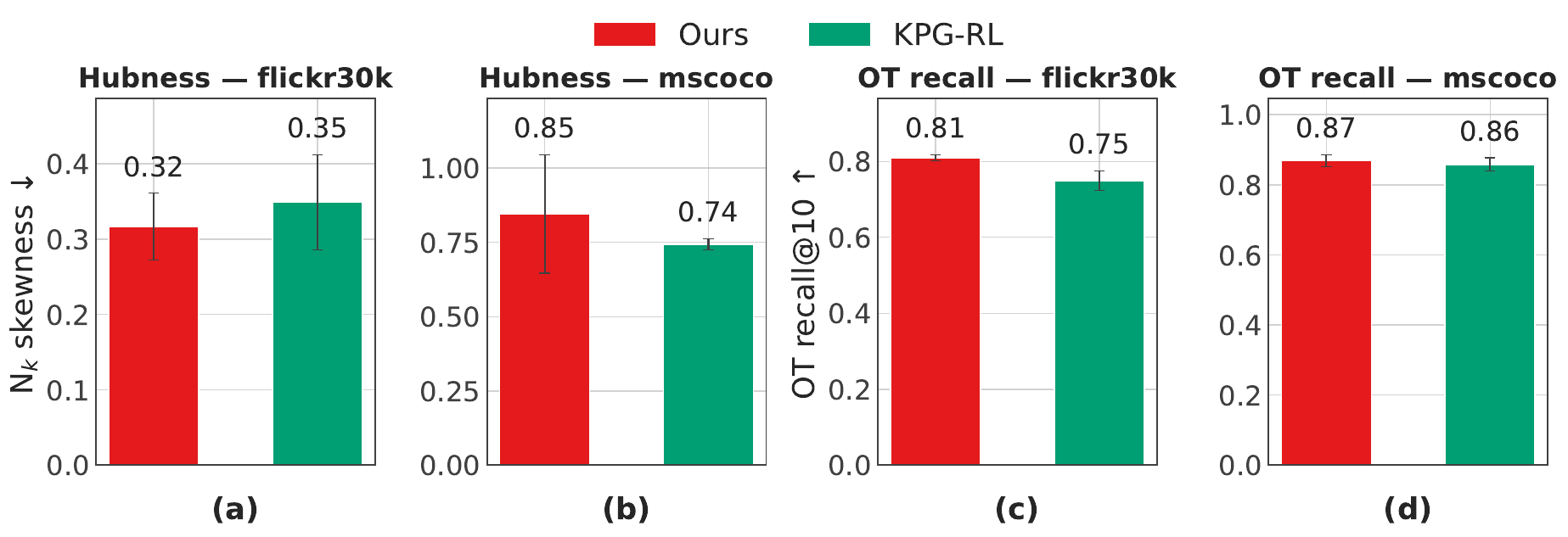}
    \caption{Hubness and recall@10 comparison on Flickr30k and MS-COCO for the KPG-RL cost against our method. Both $N_{10}$ skewness and recall@10 give similar results with our method getting slightly better performances on the retrieval task}
    \label{fig:kpg-rl}
\end{figure}

As a further alternative, we consider the relation profile of \citet{gu2022keypoint}, which also builds a cross-modal transport cost from distances to the $N$ anchors, but combines them differently. For a point $z$ in either modality, the relation score is the softmax-normalized distribution over its distances to the anchors,
\[
R_Z(z)_k = \frac{\exp(-d_Z(z, a_k)/\tau)}{\sum_{k'=1}^{N} \exp(-d_Z(z, a_{k'})/\tau)}, \qquad k \in \{1, \dots, N\},
\]
and the cross-modal cost compares the two distributions by Jensen-Shannon divergence, $C_{\text{kpg}}(x,y) = \mathrm{JS}\big(R_X(x) \,\|\, R_Y(y)\big)$. This differs from our profile representation in two respects: the anchor coordinates are pre-normalized into a probability distribution rather than compared as a raw similarity vector, and the two profiles are compared by a divergence between distributions rather than a similarity. \citet{gu2022keypoint} was designed for a different task and a far smaller anchor budget: heterogeneous domain adaptation between two feature extractors on the same visual data, guided by only $1$ to $3$ keypoints per class, rather than cross-modal retrieval with a few hundred paired anchors. It further imposes a mask constraining the coupling to match each anchor pair exactly within the transport problem itself, a hard constraint our transport plan does not impose. We plug the cost $C_{\text{kpg}}$, without the mask, into the same Sinkhorn coupling, doing the same analysis as in section \ref{sec:exp-BridgeCost} and results are reported on Figure~\ref{fig:kpg-rl}.

\subsubsection{Fused-Gromov-Wasserstein}
\label{sec:exp-FGW}

Recent critiques of the Platonic Representation Hypothesis \citep{huh2024platonic} argue that the predicted convergence holds only partially in practice, and that the two modalities may have fundamentally different intra-modal geometries \citep{groger2026aristotelian}. Our Fused Gromov-Wasserstein ablation offers a targeted piece of evidence for this view within our own setting. We interpolate between the structural and feature-cost terms via
$$
    \mathcal{L}_{\text{FGW}}(\alpha) = (1-\alpha)\, \mathcal{L}_{\text{EGW}} + \alpha\, \mathcal{L}_{\text{EOT-anchors}}, \qquad \alpha \in [0,1],
$$
where $\mathcal{L}_{\text{EGW}}$ is the entropic Gromov-Wasserstein term matching intra-modal pairwise distances, and $\mathcal{L}_{\text{EOT-anchors}}$ is our anchor-relative feature-cost objective. The structural term never improves retrieval (Figure~\ref{fig:FGW}): recall increases with $\alpha$ and is highest at or near $\alpha=1$, i.e.\ when the structural term is dropped entirely and the coupling reduces exactly to our feature-cost OT. This suggests that the intra-modal geometries of images and text are not shared closely enough to be exploited through structural alignment alone.

\begin{figure}[h]
    \centering
    \includegraphics[width=0.45\linewidth]{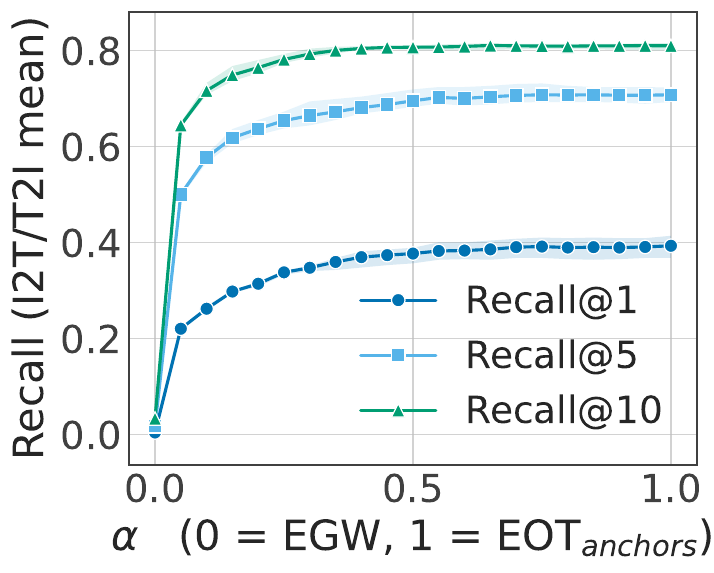}
    \caption{Recall@k as a function of the weight $\alpha$ in the FGW setting}
    \label{fig:FGW}
\end{figure}

\subsection{Qualitative cross-modal generation}
\label{sec:exp-gen}

To probe whether the shared embedding space supports generation and not
only retrieval, we reproduce the ``(almost-) text-free'' generation
protocol of \citet{yacobi2025sue} on
caption-FFHQ. A pretrained StyleGAN2 generator and its GAN-inversion
encoder (e4e, \citealp{tov2021designing}) are kept entirely fixed; we
train only a lightweight MLP \emph{converter} that maps a universal
image embedding $f_\theta(x)$ to the corresponding StyleGAN2 latent
code, using images only. At inference, a text query is instead
projected by $g_\psi$ into the same shared space and passed through
this same, image-only-trained converter and then to the GAN decoder. Generation from text alone is possible only because the coupling that aligns $f_\theta$
and $g_\psi$ makes a text embedding indistinguishable, to the converter,
from an image embedding.
\begin{figure}[h]
    \centering
    \includegraphics[width=0.9\linewidth]{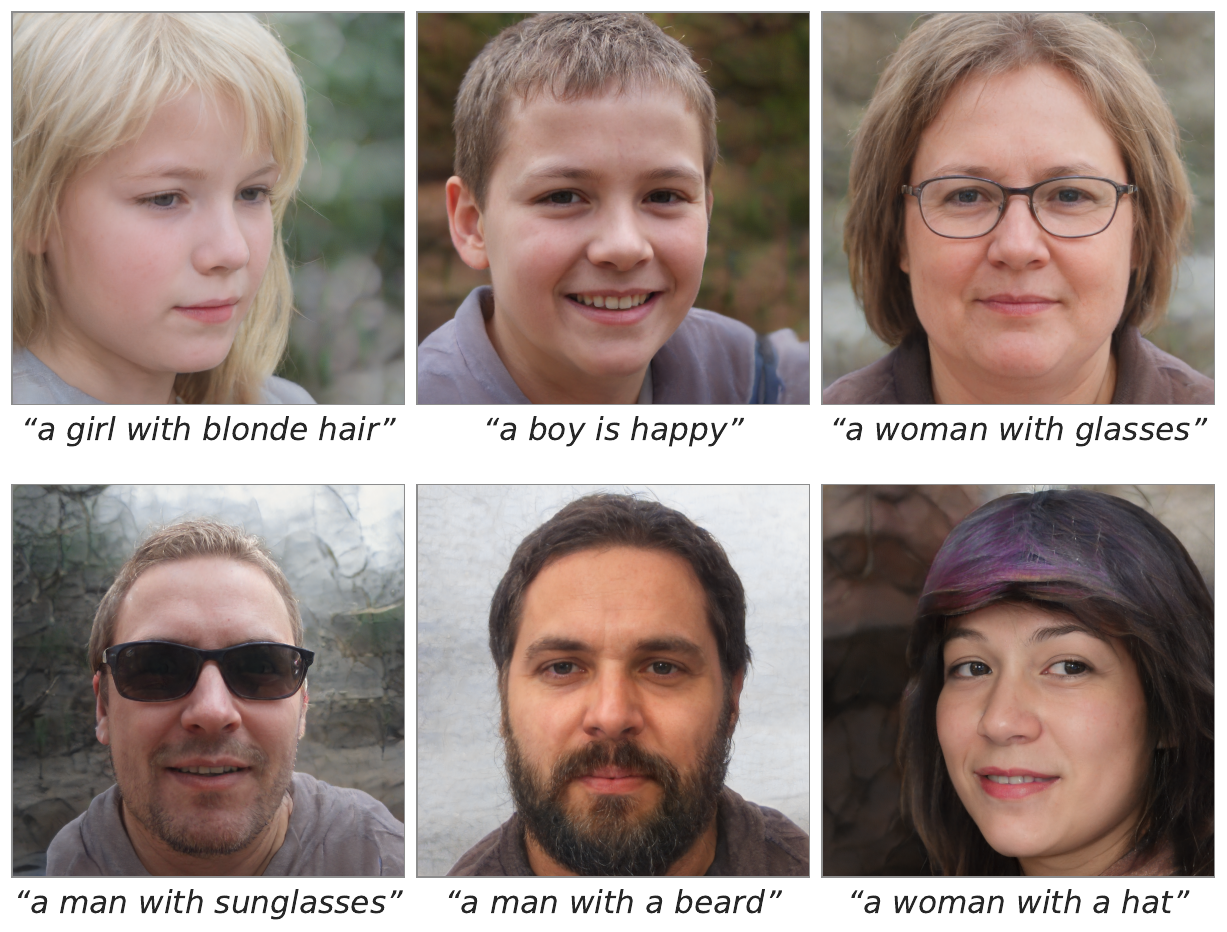}
    \caption{Examples of images generated from a custom caption following \citet{yacobi2025sue}'s generation protocol}
    \label{fig:gen}
\end{figure}

Generated faces (Figure~\ref{fig:gen}) are visually coherent and capture the coarse semantic content of the caption (hair color, expression, apparent age, glasses, facial hair), but fine or spatially localized attributes are imperfect:
captions mentioning a hat rarely yield a recognizable texture, though
the decoder consistently adds a silhouette above the head.

\subsection{Zero-shot transfer}

\subsubsection{Zero-shot classification}
\label{sec:exp-classif}

Figure~\ref{fig:classif} shows examples for the zero-shot classification experiment of Section~\ref{sec:experiments} (Table~\ref{tab:zeroshot}).

\begin{figure}[h!]
    \centering
    \includegraphics[width=0.97\linewidth]{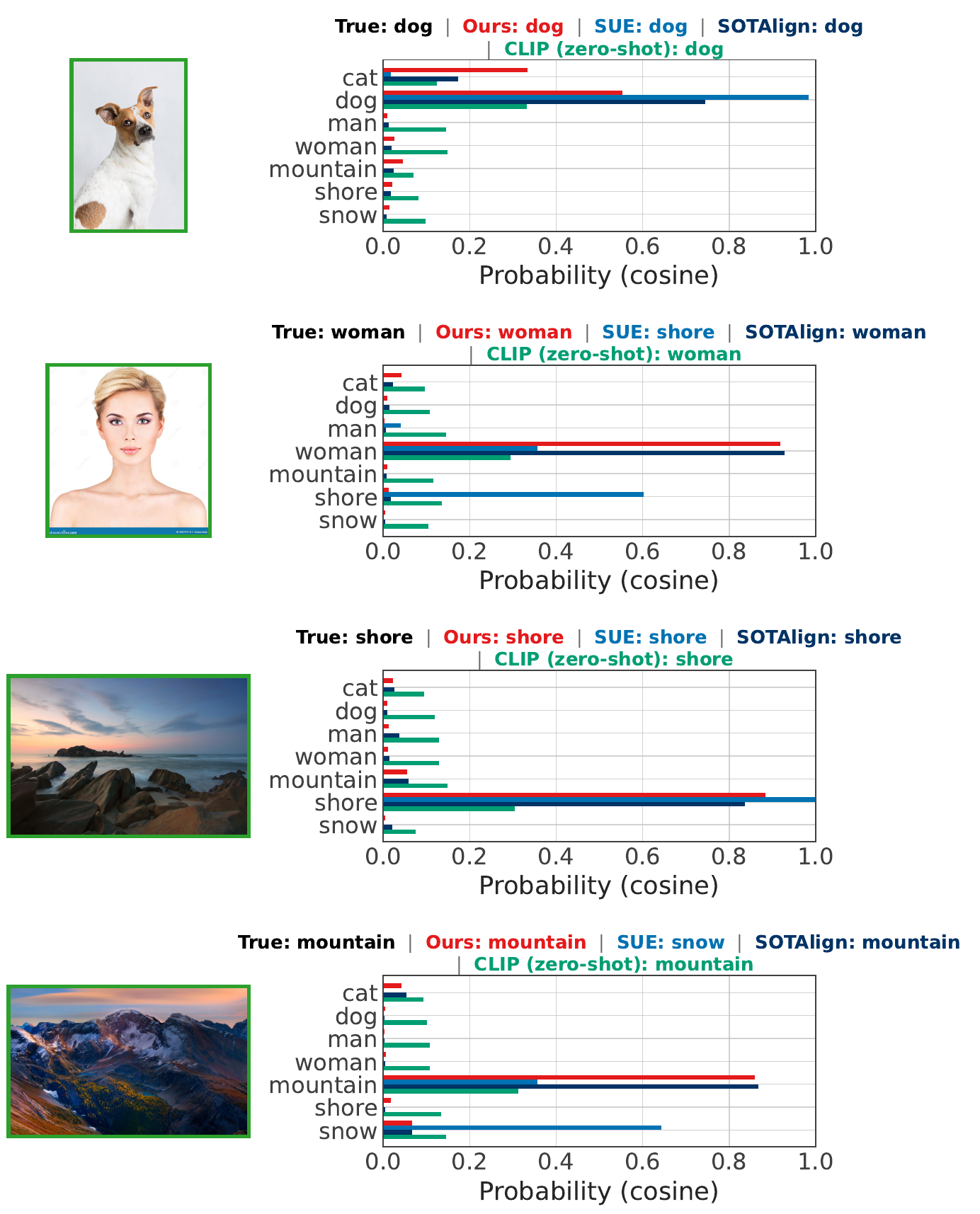}
    \caption{Examples of zero-shot classification on ImageNet comparing our method to SUE, SOTAlign (all trained on Flickr30k) and a pretrained CLIP checkpoint}
    \label{fig:classif}
\end{figure}

\clearpage
\subsubsection{Zero-shot retrieval}
\label{sec:zero_shot_retrieval}

This experiment tests whether the alignment learned by our method transfers to a new data distribution. We fit each method on $\sim 29$k
unpaired sampled from MS-COCO and $N=500$ anchors and then evaluate
recall@10 directly on Flickr30k, a dataset never seen during training or
hyperparameter selection. No further fine-tuning, calibration, or
optimization is performed on Flickr30k: the frozen unimodal encoders and
the learned alignment heads (fitted on MS-COCO) are directly applied to encode its images and captions, and retrieval is measured on a held-out query set. This is a substantially harder test than the standard held-out split, since MS-COCO and Flickr30k differ in image content, caption style, length and vocabulary.
As expected, all three methods degrade under this distribution shift but the recall remains comparable to that of Table \ref{tab:e1} (3/4 of it is kept), meaning that the alignment learned on a different distribution can be transferred to a new one. Our method remains the strongest in absolute terms in the zero-shot setting ($59.7\%$ against $56.7\%$ and $24.3\%$) , preserving the same ranking observed in-distribution ($79.3\%$ against $75.9\%$ and $37.8\%$). 

\begin{figure}[h!]
    \centering
    \includegraphics[width=0.42\linewidth]{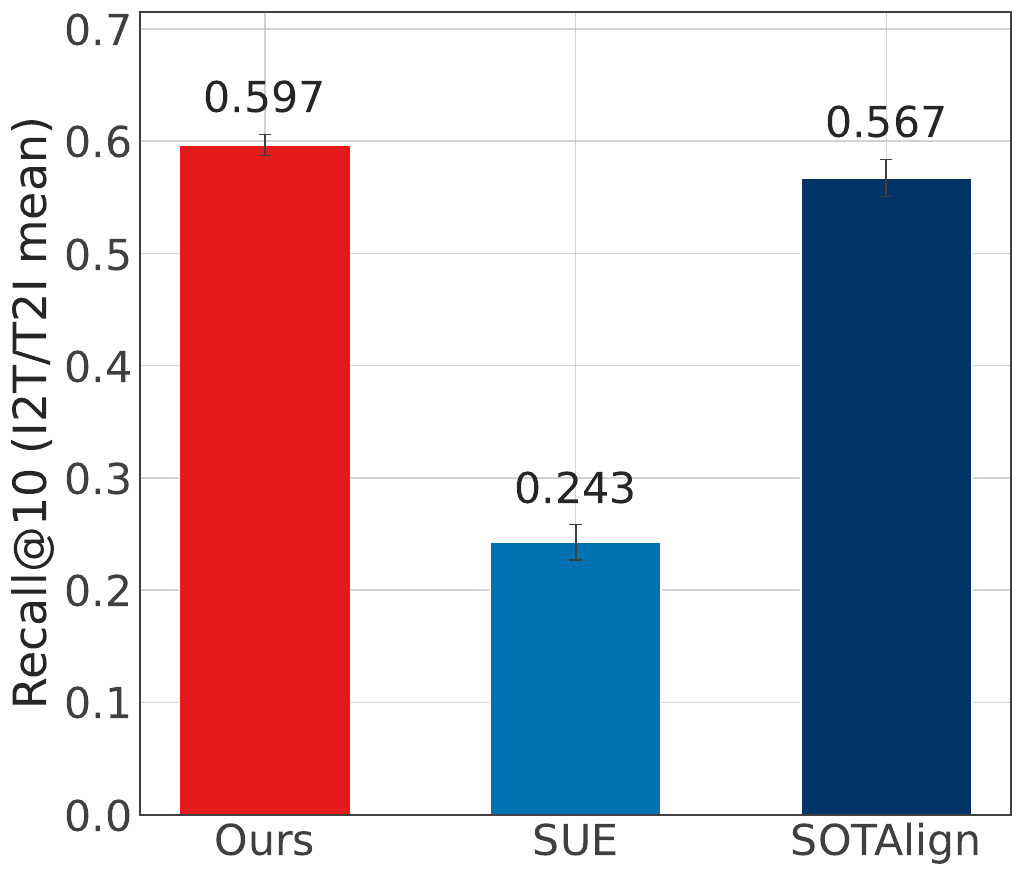}
    \caption{Comparison of our method, SUE and SOTAlign on zero-shot retrieval on Flickr30k after being trained on $30k$ samples from MS-COCO}
    \label{fig:zero_shot_retrieval}
\end{figure}

\subsection{Retrieval examples}
\label{app:retrieval_ex}

\begin{figure}[h]
    \centering
    \includegraphics[width=0.95\linewidth]{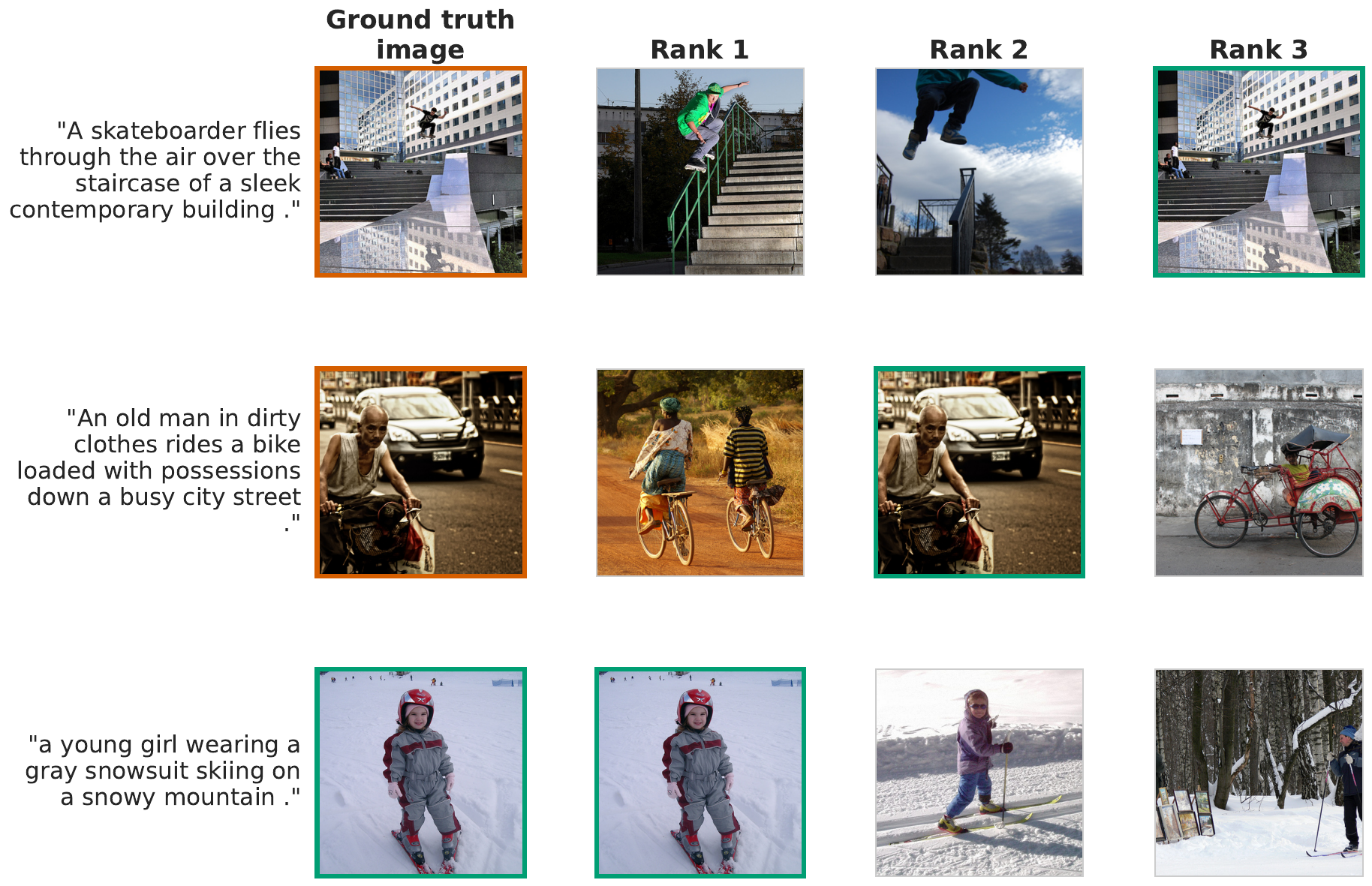}
    \caption{Text to Image retrieval examples on the dataset Flickr30k}
    \label{fig:retrieval_flick_t2i}
\end{figure}

\begin{figure}[h!]
    \centering
    \includegraphics[width=0.95\linewidth]{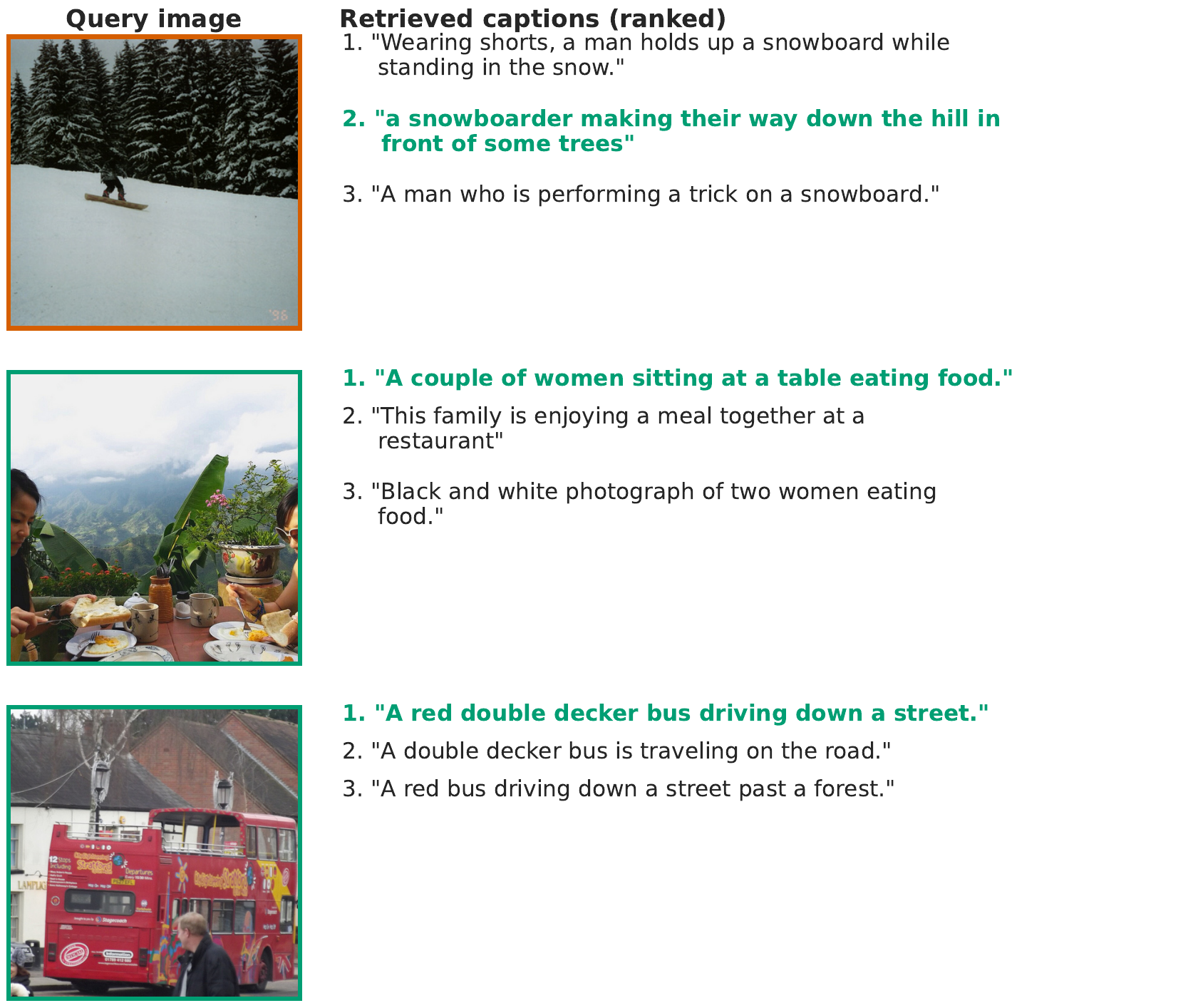}
    \caption{Image to Text retrieval examples on the dataset MS-COCO}
    \label{fig:retrieval_coco_i2t}
\end{figure}

\begin{figure}[h]
    \centering
    \includegraphics[width=0.95\linewidth]{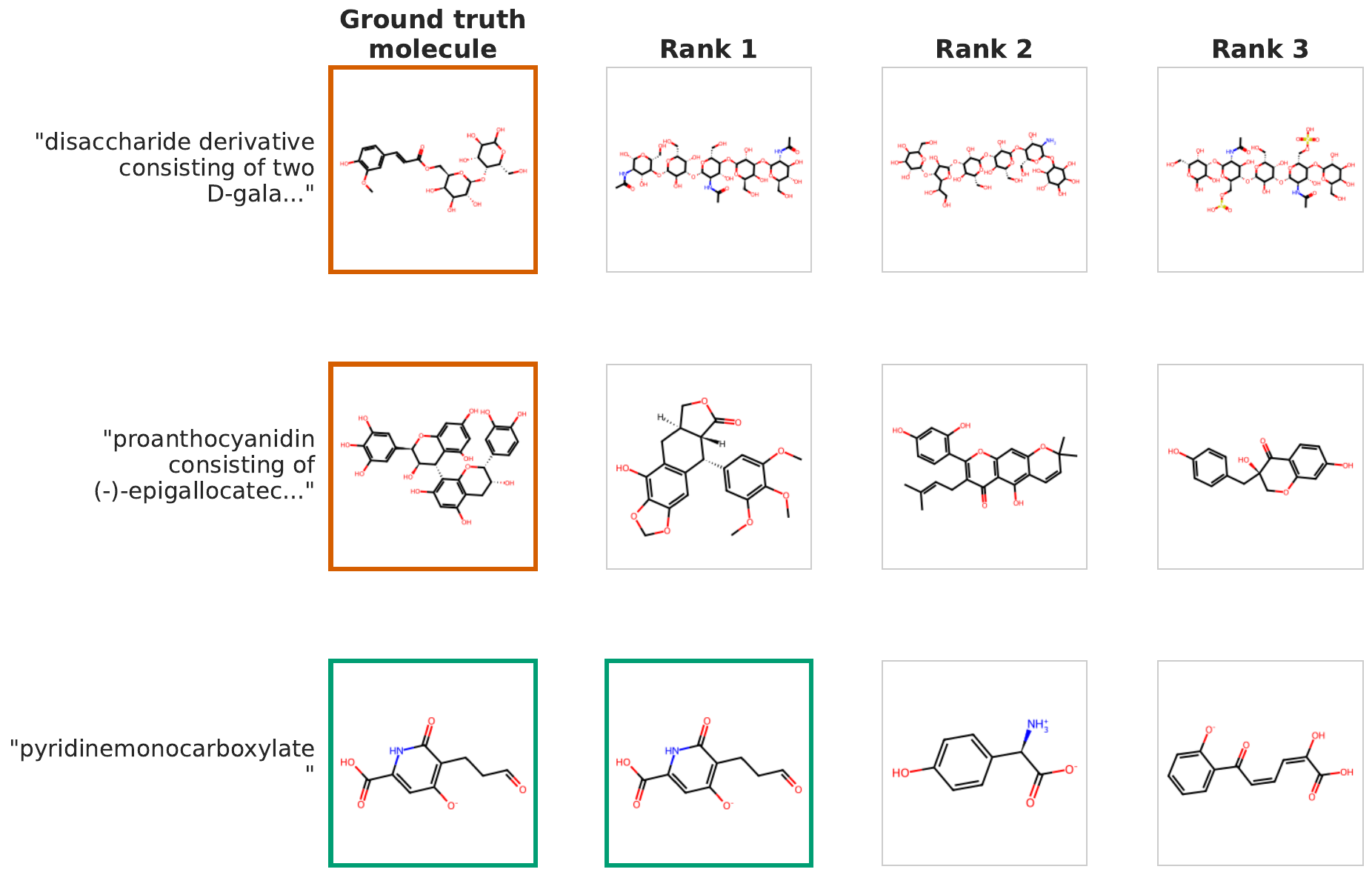}
    \caption{Text to Molecule retrieval examples on the dataset ChEBI-20}
    \label{fig:retrieval_ChEBI_t2m}
\end{figure}

\clearpage

\subsection{Complement}

\begin{table}[h!]
\caption{Complement of Table \ref{tab:e1} with std over 5 seeds for Flickr30k.}
\label{tab:e1_flickr}
\centering
\footnotesize
\setlength{\tabcolsep}{4.5pt}
\begin{tabular}{lcccccc}
\toprule
& \multicolumn{3}{c}{\textbf{\textit{Image $\rightarrow$ Text}}}
& \multicolumn{3}{c}{\textbf{\textit{Text $\rightarrow$ Image}}} \\
\cmidrule(lr){2-4}\cmidrule(lr){5-7}
\rowcolor{gray!6}
Method & R@1 & R@5 & R@10 & R@1 & R@5 & R@10 \\
\midrule
Contrastive
& 22.05 (2.21) & 49.10 (2.48) & 60.45 (2.99)
& 22.30 (2.76) & 48.05 (3.07) & 59.60 (3.20) \\

\rowcolor{gray!6}
ASIF
& 10.35 (0.57) & 29.10 (2.68) & 42.45 (3.24)
& 9.55 (0.82) & 29.80 (2.17) & 43.05 (1.56) \\

CSA
& 0.90 (0.45) & 3.15 (1.28) & 5.60 (1.27)
& 0.95 (0.65) & 3.55 (0.84) & 6.10 (0.96) \\

\rowcolor{gray!6}
SUE
& 7.80 (1.51) & 24.50 (1.71) & 37.50 (0.85)
& 7.35 (0.70) & 26.80 (1.00) & 38.05 (1.08) \\

SOTAlign
& 35.60 (2.05) & 65.00 (1.92) & 76.20 (3.29)
& 35.20 (2.31) & 65.10 (2.06) & 75.55 (2.27) \\

\midrule
\rowcolor{gray!6}
\textbf{Ours}
& \underline{\textbf{36.55} (1.66)}
& \underline{\textbf{67.90} (2.96)}
& \underline{\textbf{79.25} (2.47)}
& \underline{\textbf{36.05} (3.99)}
& \underline{\textbf{68.10} (2.23)}
& \underline{\textbf{79.40} (2.10)} \\
\bottomrule
\end{tabular}
\end{table}

\begin{table}[h!]
\caption{Complement of Table \ref{tab:e1} with std over 5 seeds for MS-COCO.}
\label{tab:e1_coco}
\centering
\footnotesize
\setlength{\tabcolsep}{4.5pt}
\begin{tabular}{lcccccc}
\toprule
& \multicolumn{3}{c}{\textbf{\textit{Image $\rightarrow$ Text}}}
& \multicolumn{3}{c}{\textbf{\textit{Text $\rightarrow$ Image}}} \\
\cmidrule(lr){2-4}\cmidrule(lr){5-7}
\rowcolor{gray!6}
Method & R@1 & R@5 & R@10 & R@1 & R@5 & R@10 \\
\midrule
Contrastive
& 14.60 (0.58) & 40.65 (2.41) & 54.95 (1.45)
& 15.00 (1.43) & 40.55 (2.69) & 54.30 (2.64) \\

\rowcolor{gray!6}
ASIF
& 15.15 (1.04) & 48.65 (2.93) & 69.95 (2.42)
& 14.20 (2.28) & 48.80 (2.26) & 69.70 (2.00) \\

CSA
& 1.30 (0.73) & 5.55 (1.47) & 9.40 (1.39)
& 1.35 (1.01) & 5.70 (1.68) & 8.95 (1.91) \\

\rowcolor{gray!6}
SUE
& 9.90 (0.72) & 36.60 (1.34) & 55.80 (3.95)
& 9.10 (1.07) & 37.25 (1.56) & 56.35 (2.31) \\

SOTAlign
& \underline{\textbf{30.50} (1.93)}
& 70.10 (1.86)
& 84.50 (2.46)
& \underline{\textbf{30.90} (1.19)}
& 69.70 (2.78)
& 84.50 (2.08) \\

\midrule
\rowcolor{gray!6}
\textbf{Ours}
& 29.05 (1.32)
& \underline{\textbf{71.05} (1.16)}
& \underline{\textbf{85.35} (2.08)}
& 30.30 (1.27)
& \underline{\textbf{70.95} (2.09)}
& \underline{\textbf{86.45} (1.60)} \\
\bottomrule
\end{tabular}
\end{table}

\begin{table}[h!]
\caption{Complement of Table \ref{tab:e2} with std over 5 seeds for ChEBI-20.}
\label{tab:e1_molecule}
\centering
\footnotesize
\setlength{\tabcolsep}{4.5pt}
\begin{tabular}{lcccccc}
\toprule
& \multicolumn{3}{c}{\textbf{\textit{Molecule $\rightarrow$ Text}}}
& \multicolumn{3}{c}{\textbf{\textit{Text $\rightarrow$ Molecule}}} \\
\cmidrule(lr){2-4}\cmidrule(lr){5-7}
\rowcolor{gray!6}
Method & R@1 & R@5 & R@10 & R@1 & R@5 & R@10 \\
\midrule

Contrastive
& 16.10 & 37.90 & 48.20
& 15.10 & 36.15 & 47.05 \\

\rowcolor{gray!6}
ASIF
& 8.40 (0.91) & 28.20 (1.52) & 40.90 (2.28)
& 14.55 (0.60) & 35.85 (1.46) & 49.40 (0.91) \\

CSA
& 0.20 (0.20) & 1.10 (0.84) & 2.55 (0.51)
& 0.65 (0.52) & 2.45 (1.08) & 4.40 (1.52) \\

\rowcolor{gray!6}
SUE
& 3.25 (0.77) & 14.85 (2.71) & 25.05 (3.83)
& 3.35 (1.01) & 15.10 (1.65) & 26.10 (3.52) \\

SOTAlign
& 20.30 (1.02) & 47.15 (2.33) & 60.60 (0.76)
& 20.30 (1.43) & 47.80 (2.03) & 62.20 (0.93) \\

\midrule
\rowcolor{gray!6}
\textbf{Ours}
& \underline{\textbf{22.10} (1.27)}
& \underline{\textbf{49.75} (1.68)}
& \underline{\textbf{63.25} (2.09)}
& \underline{\textbf{22.05} (0.92)}
& \underline{\textbf{51.00} (2.35)}
& \underline{\textbf{63.70} (1.81)} \\
\bottomrule
\end{tabular}
\end{table}

\begin{table}[h!]
\centering
\caption{Training wall-clock time (seconds) on Flickr30k, averaged over 3 seeds. We vary either the number of anchors $N$ while fixing $M=2000$, or the number of unpaired points $M$ while fixing $N=500$. Speedup is measured relative to SOTAlign.}
\label{tab:runtime}
\begin{tabular}{llrrr}
\toprule
Setting & Value & Ours & SOTAlign & Speedup \\
\midrule
\multicolumn{5}{l}{\textit{Varying number of anchors ($M=2000$)}} \\
& $N=100$  & 5.9 $\pm$ 3.4 & 77.1 $\pm$ 3.5 & 13.07$\times$ \\
& $N=500$  & 5.0 $\pm$ 0.4 & 75.0 $\pm$ 0.2 & 15.00$\times$ \\
& $N=1000$ & 7.9 $\pm$ 0.5 & 76.0 $\pm$ 0.2 & 9.62$\times$ \\
& $N=5000$ & 29.2 $\pm$ 6.8 & 99.6 $\pm$ 0.3 & 3.41$\times$ \\
\midrule
\multicolumn{5}{l}{\textit{Varying number of unpaired points ($N=500$)}} \\
& $M=2000$  & 11.05 $\pm$ 7.12 & 80.26 $\pm$ 7.76 & 7.26$\times$ \\
& $M=5000$  & 21.93 $\pm$ 1.33 & 305.99 $\pm$ 0.75 & 13.95$\times$ \\
& $M=15000$ & 180.0 $\pm$ 23.67 & 2305.78 $\pm$ 4.06 & 12.81$\times$ \\
\bottomrule
\end{tabular}
\end{table}

\section{Hyperparameters}
\label{app:hyperparam}

Beyond the anchor budget $N$ and the pool size, our method involves
three groups of hyperparameters: standard encoder-training
hyperparameters, the entropic regularization of the Sinkhorn coupling,
and the anchor-relative profile itself. The first
group is fixed across datasets and seeds and summarized in
Table~\ref{tab:hyperparams-training}; only some of the Sinkhorn and profile hyperparameters depend on the seed and dataset and are selected with the held-out procedure described below.

\begin{table}[h]
\centering
\caption{Training hyperparameters, fixed across datasets
and seeds.}
\label{tab:hyperparams-training}
\begin{tabular}{lc}
\toprule
Hyperparameter & Value \\
\midrule
Projection dimension            & 1024 \\
Optimizer                       & Adam \\
Learning rate                   & $10^{-3}$ \\
Weight decay                    & $10^{-2}$ \\
Warmup steps                    & 200 \\
Epochs                          & 200 \\
Temperature initialization      & $0.07$ \\
Chunk size                      & 2000 \\
Sinkhorn max iterations         & 1000 \\
Sinkhorn tolerance              & 0.0001 \\
\bottomrule
\end{tabular}
\end{table}

\paragraph{Sinkhorn regularization.} The Sinkhorn coupling built on the
anchor-relative profile is regularized by $\varepsilon = \eta \cdot
\mathrm{median}(C)$, where $C$ is the anchor-relative cost matrix and
$\eta$ a single scalar multiplier, selected per dataset and seed
(Table~\ref{tab:hyperparams-selected}). Because $\mathrm{median}(C)$ is
recomputed from data every time, $\eta$ alone suffices to adapt
$\varepsilon$ to the scale of $C$ for every dataset and every candidate
profile -- $\varepsilon$ itself is never tuned directly.

\paragraph{Profile hyperparameters.} The anchor-relative profile is
built from $\gamma$-interpolated orthonormalized embeddings: each
modality is whitened independently before a point is compared to each
anchor by cosine similarity, with $\gamma \in [0,1]$ interpolating
between the untouched profile ($\gamma=0$) and an unmodified spectrum ($\gamma=1$). Together with $\eta$ above, $\gamma$ is the
only other hyperparameter we select per dataset and seed
(Table~\ref{tab:hyperparams-selected}); every other aspect of the
profile construction is fixed across datasets.

\paragraph{Selecting the profile hyperparameters.} For every seed, we split the $N$ anchors $90/10$: the larger split defines the anchor-relative
representation, while the OT coupling is solved, for each candidate
$(\eta, \gamma)$ on a small grid, over the unpaired pool together with
the held-out $10\%$ of anchors, with everything else about the profile
fixed, so that only $(\eta, \gamma)$ vary between candidates. For every
held-out anchor we record the rank of its true match in the resulting
coupling, in both directions, and keep the candidate minimizing the
mean held-out rank (Fig.~\ref{fig:hyperparam}). With
$(\eta, \gamma)$ fixed, we use this same configuration for the
corresponding distillation run, over the full set of $N$ anchors.
Table~\ref{tab:hyperparams-selected} reports the resulting choice per
dataset for one seed (results are very consistent across different seeds of the same dataset but may vary).

\begin{figure}[h]
    \centering
    \includegraphics[width=0.99\linewidth]{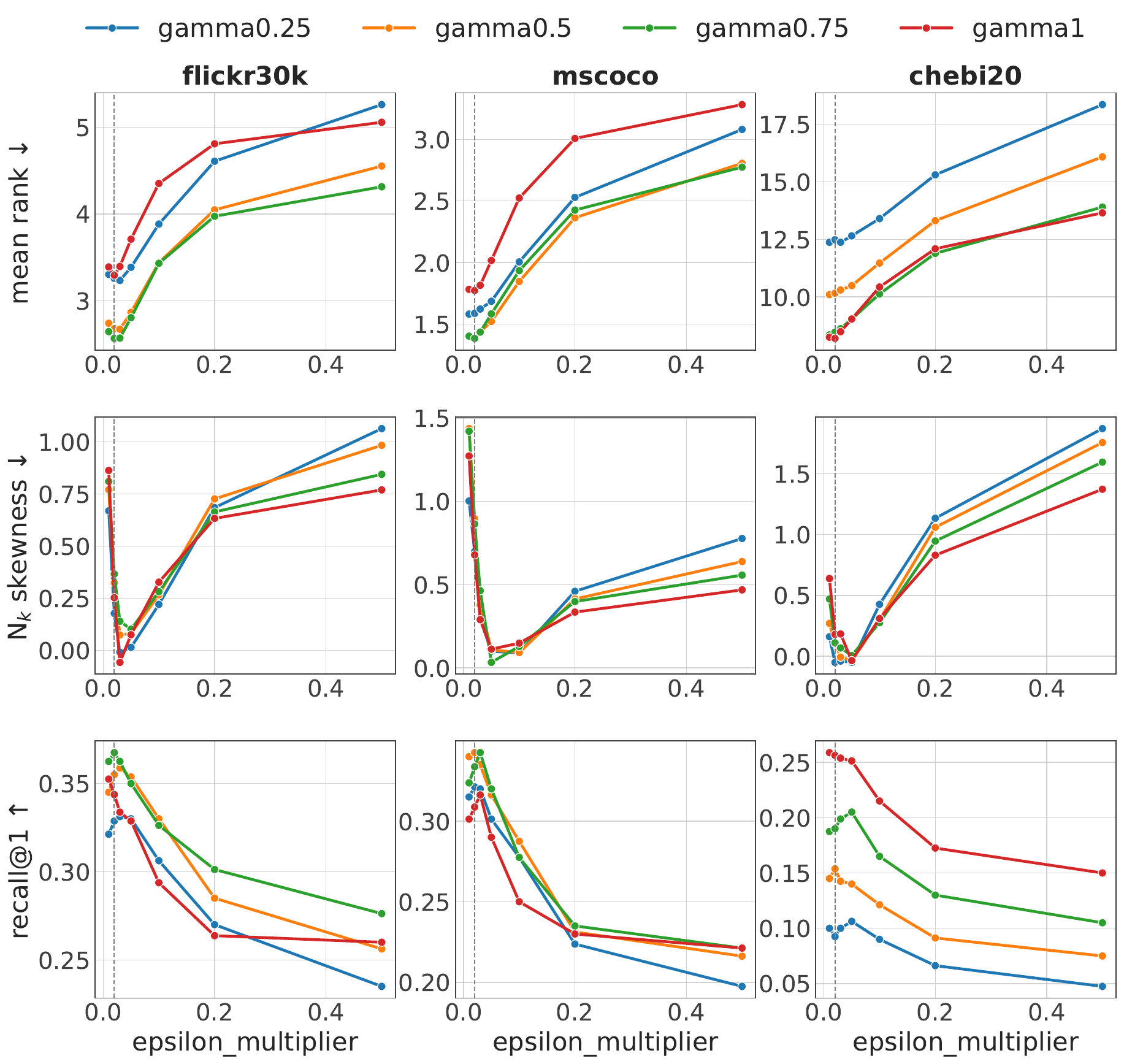}
    \caption{Held-out rank (mean, both retrieval directions), $N_{10}$ skewness and recall@1 as a function of the epsilon multiplier $\eta$, one curve per $\gamma$ value. The selected $(\gamma^\star, \eta^\star)$ minimizes this
    held-out rank.}
    \label{fig:hyperparam}
\end{figure}

\begin{table}[b!]
\centering
\caption{Selected profile hyperparameters per dataset (seed 0)}
\label{tab:hyperparams-selected}
\begin{tabular}{lccc}
\toprule
Dataset & $\gamma^\star$ & $\eta^\star$ \\
\midrule
Flickr30k & 0.75 & 0.02 \\
MS-COCO   & 0.75 & 0.02 \\
ChEBI-20  & 1.00 & 0.02 \\
\bottomrule
\end{tabular}
\end{table}

\end{document}